%% file: root.tex
\documentclass[letterpaper, 10 pt, conference]{ieeeconf}  %

\IEEEoverridecommandlockouts                              %

\usepackage{graphics} %
\usepackage{amsmath} %
\usepackage{amssymb}  %

\usepackage{url}
\usepackage{bbm}
\usepackage{todonotes}
\usepackage{pifont}
\usepackage{amsthm}
\usepackage{caption}
\usepackage{subcaption}
\usepackage{cite}
\usepackage{tikz}
\usepackage{algorithm}
\usepackage{algorithmic}
\usepackage{relsize}
\newtheorem{theorem}{Theorem}
\newtheorem{lemma}{Lemma}
\newtheorem{corollary}{Corollary}
\newtheorem{definition}{Definition}
\newtheorem{problem}{Problem}
\newtheorem{remark}{Remark}
\usetikzlibrary{snakes,arrows,shapes}

\makeatletter
\renewenvironment{proof}[1][\proofname]{\par
\pushQED{\qed}%
\normalfont \topsep6\p@\@plus6\p@\relax
\trivlist
\item\relax
{\itshape
#1\@addpunct{.}}\hspace\labelsep\ignorespaces
}{%
\popQED\endtrivlist\@endpefalse
}
\newcommand\fs@betterruled{%
  \def\@fs@cfont{\bfseries}\let\@fs@capt\floatc@ruled
  \def\@fs@pre{\vspace*{4pt}\hrule height.8pt depth0pt \kern2pt}%
  \def\@fs@post{\kern2pt\hrule\relax\vspace*{-9pt}}%
  \def\@fs@mid{\kern2pt\hrule\kern2pt}%
  \let\@fs@iftopcapt\iftrue}
\floatstyle{betterruled}
\restylefloat{algorithm}
\makeatother
\title{\LARGE \bf
Control Synthesis from Linear Temporal Logic Specifications using Model-Free Reinforcement Learning
}

\author{Alper Kamil Bozkurt, Yu Wang, Michael M. Zavlanos, and Miroslav Pajic%
\thanks{This work is sponsored in part by the ONR under agreements N00014-17-1-2504, N00014-20-1-2745 and N00014-18-1-2374, AFOSR award number FA9550-19-1-0169, and the NSF CNS-1652544 and CNS-1932011 grants.}%
\thanks{Alper Kamil Bozkurt, Yu Wang, Michael M. Zavlanos and Miroslav Pajic are with Duke University, Durham, NC 27708, USA, {\tt\small \{alper.bozkurt, yu.wang094, michael.zavlanos, miroslav.pajic\}@duke.edu}.}%
}

\begin{document}

\maketitle
\thispagestyle{empty}
\pagestyle{empty}

\newcommand{\yw}[1]{{\color{red} #1}}

\begin{abstract}
We present a reinforcement learning (RL) framework to synthesize a control policy from a given linear temporal logic (LTL) specification in an unknown stochastic environment that can be modeled as a Markov Decision Process (MDP). Specifically, we learn a policy that maximizes the probability of satisfying the LTL formula without learning the transition probabilities. We introduce a novel rewarding and discounting mechanism based on the LTL formula such that (i) an optimal policy maximizing the total discounted reward effectively maximizes the  probabilities of satisfying LTL objectives, and (ii) a model-free RL algorithm using these rewards and discount factors is guaranteed to converge to such a policy. Finally, we illustrate the applicability of our RL-based synthesis approach on two motion planning case studies. 
\end{abstract}

\input{intro}

\input{prelim}

\input{problem}

\input{buchi}

\input{experiment}

\vspace{-2pt}
\section{Conclusion} \label{sec:conc}
\vspace{-3pt}
In this work, we present a model-free learning-based method to synthesize  a control policy that \emph{maximizes} probability that an LTL specification is satisfied in \emph{unknown} stochastic environments that can be modeled by an MDP.
We first show that synthesizing controllers from an LTL specification on the MDP can be converted to synthesizing a memoryless policy of a B\"uchi objective on the product MDP.
Then, we design a novel discounting and reward scheme, and show that the memoryless policy optimizing this reward, also optimizes the satisfaction probability of the B\"uchi objective (and thus the initial LTL specification).   
Finally, we evaluate our synthesis method on motion planning case~studies.

\iffalse 
\section*{APPENDIX}

Appendixes should appear before the acknowledgment.

\section*{ACKNOWLEDGMENT}

References are important to the reader; therefore, each citation must be complete and correct. If at all possible, references should be commonly available publications.

\fi

\bibliography{references}
\bibliographystyle{unsrt}

\end{document}

%% file: intro.tex
\section{Introduction}

Formal logics have been used to facilitate robot motion planning beyond its traditional focus on computing robot trajectories that, starting from an initial region, reach a desired goal without hitting any obstacles (e.g.,~\cite{karaman2011sampling, karaman2011anytime}). 
Linear Temporal Logic (LTL) is a widely used framework for formal specification of high-level robotic tasks on discrete~models. 
Thus, control synthesis on discrete-transition systems for LTL objectives has attracted a lot of attention (e.g.,~\cite{vasile2013,smith2011,chen2012,kantaros2017,wolff2014}).

Another line of work considers motion planning for LTL objectives for systems that exhibit uncertainty coming from either robot dynamics or the environment, such as  Markov Decision Processes (MDPs)~\cite{guo2018,guo2015multi,kantaros2019,lahijanian2012,wolff2012, kwiatkowska2013, ding2014}.
By synthesizing control for an MDP, from a given LTL objective, %
the obtained  controller maximizes the probability of satisfying the specification. 
Also, tools from probabilistic model checking~\cite{baier2008} can be directly used for synthesis. 
Yet, when the MDP transition probabilities are not known a priori, the control policy needs to be synthesized through learning from~samples.

Hence, there is a recent focus on learning for control (i.e., motion planning) synthesis for LTL objectives~\cite{fu2014,brazdil2014,wen2015,aksaray2016,li2017,toro2018,degiacomo2019,sadigh2014,gao2019,hahn2019,li2018policy}. Most model-based reinforcement learning (RL) methods are based on detection of end components, and provide estimates of satisfaction probabilities with probably approximately correct %
bounds (e.g.,~\cite{fu2014,brazdil2014}). Such approaches, however, need to first learn and store the MDP transition probabilities, and thus have significant space~requirements, restricting their use to systems with small and low-dimensional state spaces.

On the other hand, model-free %
RL methods derive the desired policies without storing a model of the MDP. The temporal logic tasks need to be represented by a reward function, possibly with a finite-memory, so that the optimal policy maximizing the discounted future reward, also maximizes the probability of satisfying the tasks. One approach is to use temporal logic specifications that are time-bounded or defined on finite traces so that they can be directly translated to a real-valued reward function~\cite{aksaray2016,li2017,toro2018,degiacomo2019}. Alternatively, unbounded LTL formulas can be transformed into an $\omega$-automaton and the accepting condition of the automaton can be used to design the reward function. 

Such reward functions based on Rabin conditions are introduced in~\cite{sadigh2014}, as part of a model-free RL method; the approach assigns a sufficiently small negative and a positive reward to the first and second sets of the Rabin pairs, respectively. 
A generalization %
to deep Q-learning, with a new optimization algorithm, is done in~\cite{gao2019}. Yet, in the presence of rejecting end components or multiple Rabin pairs, optimal policies {obtained by this method} may not satisfy the LTL property almost surely, even if such policies exist~\cite{hahn2019}.

A given LTL property can also be translated into a limit-deterministic B\"uchi automaton (LDBA), which can be used in quantitative analysis of MDPs~\cite{hahn2015,sickert2016}. The first reward function based on LDBA accepting conditions is introduced in \cite{hasanbeig2018}. Yet, similar to~\cite{sadigh2014}, in the presence of non-accepting components, the algorithm might fail to converge to the policy that almost surely satisfies the LTL specification. 

The problem of satisfying the B\"uchi condition of an LDBA can be reduced to a reachability problem by adding transitions with a positive reward from accepting states to a terminal state~\cite{hahn2019}. 
Then, as the transition probability from an accepting state to the terminal state goes to zero, in order to reach the terminal state and obtain a positive reward, an accepting state needs to be visited infinitely often, which captures the satisfaction of the B\"uchi condition.
However, model-free RL algorithms such as Q-learning may fail to converge to the correct reachability probabilities without discounting (or improper discounting) in the presence of end components~\cite{brazdil2014}, as Q-learning might get stuck in one of the fixed-point solutions; e.g., if all the values are initialized to 1, Q-learning will not be able to decrease any value estimate. 

Consequently, in this paper, we propose a model-free RL algorithm that is \emph{guaranteed to find a control policy that maximizes the  probability of satisfying a given LTL objective (i.e., specification) in an arbitrary unknown MDP}; for the MDP,
not even which probabilities are nonzero (i.e., its graph/topology) is known. We use an automata-based approach %
that constructs a product MDP using an LDBA of a given LTL formula and assigns rewards based on the B\"uchi (repeated reachability) acceptance condition. Such optimal policy can then be derived by learning a policy maximizing the satisfaction probability of the B\"uchi condition on the product. Unlike~\cite{hahn2019}, our approach directly assigns positive rewards to the accepting states and discounts these rewards in such a way that the values of the optimal policy are \emph{proved to converge to the maximal satisfaction probabilities} as the discount factor %
goes beyond a threshold that is less than $1$.

The rest of the paper is organized as follows. We introduce preliminaries 
and formalize the problem in Sec.~\ref{sec:prelim}.
Sec.~\ref{sec:buchi} presents our model-free RL algorithm that maximizes probabilities that LTL specifications are satisfied. %
Finally, we evaluate our approach on several motion planning problems for mobile robots (Sec.~\ref{sec:case}), before concluding in Sec.~\ref{sec:conc}.

%% file: prelim.tex
\section{Preliminaries and Problem Statement} \label{sec:prelim}
We start with preliminaries on LTL, MDPs, and RL on MDPs, before problem formulation.
We denote the sets of real and natural numbers by $\mathbb{R}$ and $\mathbb{N}$, respectively.
For a set $S$,  $S^+$ denotes the set of all finite sequences taken from~$S$. 

\subsection{Markov Decision Processes and Reinforcement Learning} \label{sec:mdp}

MDPs are common modeling formalism for systems that permit nondeterministic choices with probabilistic outcomes. 
\begin{definition} \label{def:mdp}
A (labeled) MDP is a tuple $\mathcal{M}=(S, A, P, s_0, \text{AP}, \allowbreak L)$, where $S$ is a finite set of states, $A$ is a finite set of actions, $P: S \times A \times S \to [0,1]$ is the transition probability function, $s_0 \in S$ is an initial state, $\text{AP}$ is a finite set of atomic propositions, and $L: S \to 2^{\text{AP}}$ is a labeling function. For simplicity, let $A(s)$ denote the set of actions that can be taken in state $s$; then for all states $s \in S$, it holds that $\sum_{s' \in S} P(s, a, s')=1$ if $a\in A(s)$, and $0$ otherwise.
\end{definition}

A path is an infinite sequence of states $\sigma=s_0s_1s_2\dots$, with $s_i\in S$ such that for all $i \geq 0$, there exists $a_i \in A$ with $P(s_i,a_i,s_{i+1})>0$.
We use $\sigma[i]$ to denote the state $s_i$, as well as $\sigma[{:}i]$ and $\sigma[i{+}1{:}]$ to denote the prefix $s_0s_1\dots s_{i}$ and the suffix $s_{i+1}s_{i+2}\dots$ of the path, respectively. 

\begin{definition}
A \textbf{policy} $\pi$ for an MDP $\mathcal{M}$ is a function $\pi:S^+\to A$ such that $\pi(\sigma[{:}n])\in A(\sigma[n])$. 
A policy is \textbf{memoryless} if it only depends on the current state, i.e., $\pi(\sigma[{:}n]) = \pi(\sigma[n])$ for any $\sigma$, and thus can be defined as $\pi: S \to A$.
A Markov chain (MC) of an MDP $\mathcal{M}$ induced by a memoryless policy $\pi$ is a tuple $\mathcal{M}_\pi=(S,P_\pi,s_0,\text{AP},L)$,
where $P_\pi(s, s') = P(s, \pi(s), s')$ for all $s, s' \in S$.
A \textbf{bottom strongly connected component} (BSCC) of an MC is a strongly connected component with no outgoing transitions.
\end{definition}

Let $R: S\to \mathbb{R}$ be a \emph{reward function} of the MDP $\mathcal{M}$. 
Then, for a discount factor $\gamma \in (0, 1)$, the $K$-step \emph{return} ($K \in \mathbb{N}$ or $K = \infty$) of a path $\sigma$ from time $t \in \mathbb{N}$ is 
\vspace{-8pt}
\begin{align}
    G_{t{:}K}(\sigma) = \sum_{i=0}^{K} \gamma^{i}R(\sigma[t{+}i]), ~ G_t(\sigma) = \lim_{K\to\infty} G_{t{:}K}(\sigma). \label{eq:return}
\end{align}
\vspace{-8pt}

\noindent Under a policy $\pi$, the \emph{value} of a state $s$ is defined as the expected return of a path -- i.e.,
\begin{align} 
    v_\pi(s) = \mathbb{E}_\pi\left[ G_t(\sigma) \mid \sigma[t]=s \right], \label{eq:expected return}
\end{align}
for any fixed $t \in \mathbb{N}$ such that $Pr_\pi^\mathcal{M}(\sigma[t] = s ) > 0$.

The RL objective is to find an optimal policy $\pi^*$ for MDP $\mathcal{M}$ from samples, such that the return $v_\pi(s)$ is maximized for all $s \in S$; we denote the maximum by $v_*(s)$.
Specifically, RL is \textit{model-free}, if $\pi^*$ is derived without~explicitly~estimating the transition probabilities, as done in model-based~RL; hence, it scales significantly better in large applications~\cite{strehl2006}.

\vspace{-1pt}
\subsection{LTL and Limit-Deterministic B\"uchi~Automata} \label{sec:ltl_prelim}
\vspace{-1pt}
LTL provides a high-level language to describe specifications of a system. LTL formulas can be constructed inductively as combinations of Boolean operators, negation~($\neg$) and conjunction ($\wedge$), and two temporal operators, next ($\bigcirc$) and until ($\textsf{U}$), using the following syntax:
\vspace{-2pt}
\begin{align}
    \varphi ::= \mathrm{true} \mid a \mid \varphi_1 \wedge \varphi_2 \mid \neg \varphi \mid \bigcirc \varphi \mid \varphi_1 \textsf{U} \varphi_2, ~ {a\in\text{AP}} \label{eq:ltl}.
\end{align}

The satisfaction of an LTL formula $\varphi$ for a path $\sigma$ of an MDP from Def.~\ref{def:mdp} (denoted by $\sigma \models \varphi$) is defined as follows:
$\sigma$ satisfies an atomic proposition $a$, if the first state $s_0$ of the path is labeled with $a$, %
i.e., $a \in L(s_0)$;
a path $\sigma$ satisfies $\bigcirc \varphi$ if $\sigma[1{:}]$ satisfies the formula $\varphi$; %
and finally, 
\vspace{-2pt}
\begin{align}
    \sigma \models \varphi_1 \textsf{U} \varphi_2, ~~ \text{if }\ \exists i. \sigma[i] \models \varphi_2 \text{ and } \forall j<i. \sigma[j] \models \varphi_1 \label{eq:until}.
\end{align}
Other common Boolean and temporal operators are derived as follows: 
(or) $\varphi_1 \lor \varphi_2 \equiv \neg (\neg \varphi_1 \land \neg \varphi_2)$;
(implies) $\varphi_1 \to \varphi_2 \equiv \neg \varphi_1 \lor \varphi_2$;
(eventually) $\lozenge \varphi \equiv \mathrm{true}\ \textsf{U}\ \varphi$; and (always) $\square \varphi \equiv \neg (\lozenge \neg \varphi)$~\cite{baier2008}.

Satisfaction of an LTL formula can be evaluated on a Limit-Deterministic B\"uchi~Automata (LDBA) that can be directly derived from the formula~\cite{hahn2015,sickert2016}.

\begin{definition} \label{def:ldba}
An %
LDBA is a tuple $\mathcal{A}=(Q, \Sigma, \allowbreak \delta, q_0, B)$, where $Q$ is a finite set of states, $\Sigma$ is a finite alphabet, $\delta: Q \times (\Sigma \cup \{\epsilon\}) \to 2^Q$ is a (partial) transition function, $q_0 \in Q$ is an initial state, and $B$ is a set of \emph{accepting states}, such that (i)~$\delta$ is total except for the $\epsilon$-moves, i.e., $|\delta(q, \alpha)| = 1$ for all $q \in Q, \alpha \in \Sigma$; and (ii)~there exists a bipartition of $Q$ to an initial and an accepting component, i.e., $Q_I \cup Q_A = Q$,~where
\begin{itemize}
    \item the $\epsilon$-moves are not allowed in the accepting component, i.e., 
    for any $q \in Q_A$, $\delta(q, \epsilon) = \emptyset$, %
    \item outgoing transitions from the accepting component stay within it, i.e., for any $q \in Q_A, \nu \in \Sigma$, $\delta(q, \nu) \subseteq Q_A$,
    \item the accepting states are in the accepting component, i.e., $B \subseteq Q_A$. 
\end{itemize}
An infinite path $\sigma$ is accepted by the LDBA if it satisfies the \textbf{B\"uchi condition} -- i.e., $\text{inf}(\sigma) \cap B \neq \emptyset$, where $\text{inf}(\sigma)$ denotes the set of states visited by $\sigma$ infinitely many times.
\end{definition}

%% file: problem.tex
\subsection{Problem Statement} \label{sec:problem}
\vspace{-2pt}
In this work, we consider the problem of synthesizing a robot control policy in a stochastic environment such that \textbf{\textit{the probability of satisfying the desired specification is maximized}}.
The robot environment is modeled as an MDP with \emph{unknown transition probabilities} (i.e., not even which probabilities are nonzero is known), and the desired objective (i.e., specification) is given by an LTL formula. Our goal is to obtain such a policy by \emph{learning the maximal probabilities that the LTL specification is satisfied}; this should be achieved by directly interacting with the environment -- i.e.,  \emph{without constructing a model of the~MDP}.

For any policy $\pi$, $Pr_\pi(s \models \varphi)$ denotes the probability~of all paths from the state $s$ to satisfy formula $\varphi$ under~the~policy %
\begin{align}
    Pr_\pi^\mathcal{M}(s \models \varphi) := Pr_\pi^\mathcal{M}\left\{\sigma \mid \sigma[0]=s \text{ and } \sigma \models \varphi \right\}.
\end{align}
We omit the superscript $\mathcal{M}$ when it is clear from the context. We now formally state the problem considered in this work.
\begin{problem} \label{problem}
Given an MDP $\mathcal{M}=(S,A,P,s_0,\allowbreak \text{AP},L)$ where $P$ is fully
unknown and an LTL specification $\varphi$, design a model-free RL algorithm that finds a finite-memory objective policy $\pi^\varphi$ that satisfies
\begin{align}
    Pr_{\pi^\varphi}\left(s \models \varphi \right) = Pr_\text{max}(s \models \varphi),\label{eq:theorem1_pi_max}
\end{align}
where $Pr_\text{max}(s \models \varphi) := \max_{\pi} Pr_\pi(s \models \varphi)$ for all $s\in S$.
\end{problem}

%% file: buchi.tex
\section{RL-Based Synthesis from LTL Specifications} \label{sec:buchi}

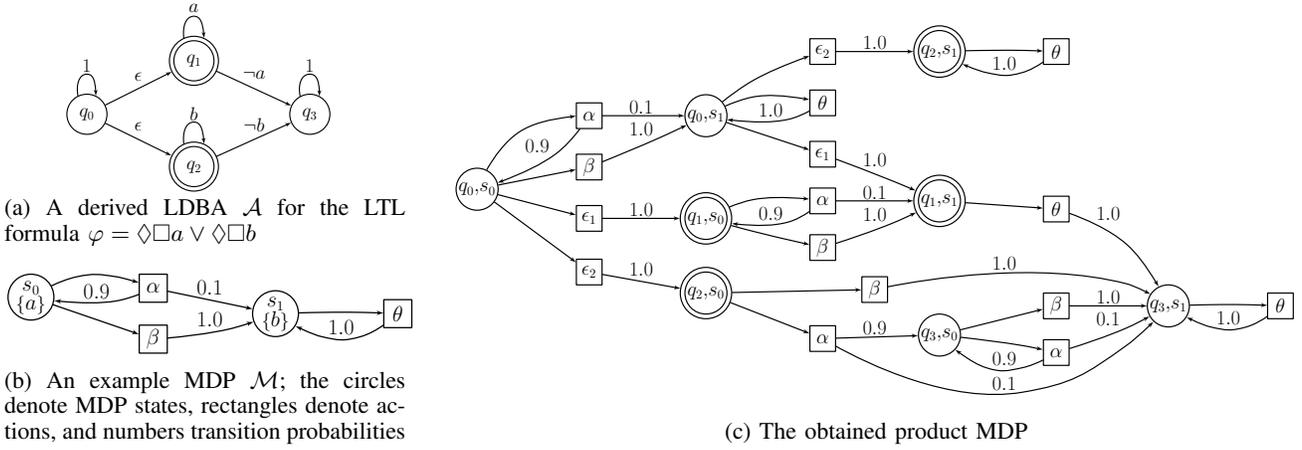
\begin{figure*}
\vspace{-2pt}
\centering
\begin{subfigure}[b]{0.3\textwidth}
    \begin{subfigure}[b]{\textwidth}
    \vspace{4pt}
        \centering
        \scalebox{.42}{\input{graph1.tex}}
        \vspace{-4pt}
        \caption{A derived LDBA $\mathcal{A}$ for the LTL formula $\varphi=\lozenge \square a \vee \lozenge \square b$}
        \label{fig:ldba}
        \vspace{0.92em}
    \end{subfigure}
    \begin{subfigure}[b]{\textwidth}
        \centering
        \scalebox{.48}{\input{graph2.tex}}
        \caption{An example MDP $\mathcal{M}$; the circles denote MDP states, rectangles denote  actions, and numbers transition probabilities}
        \label{fig:mdp}
    \end{subfigure}
\end{subfigure}
\begin{subfigure}[b]{0.69\textwidth}
    \centering
    \scalebox{.44}{\input{graph3.tex}}
    \caption{The obtained product MDP}
    \label{fig:product_mdp}
\end{subfigure}
\caption{Product MDP $\mathcal{M}^\times$ obtained from an MDP $\mathcal{M}$ and an LDBA $\mathcal{A}$ that is automatically derived from an LTL formula~$\varphi$.}
 \vspace{-9pt}
\label{fig:product_mdp_construction}
\end{figure*}

In this section, we introduce a framework to solve Problem~\ref{problem}. %
We start by exploiting the fact that any LTL formula can be transformed into an LDBA that can be used in quantitative analysis of MDPs~\cite{hahn2015,sickert2016}; in such LDBAs, the only nondeterministic transitions are the $\epsilon$-moves from the initial component to the accepting component (e.g., see Fig.~\ref{fig:ldba}). Therefore, we reduce the problem of satisfying a given LTL objective $\varphi$ in an MDP $\mathcal{M}$ to the problem satisfying a repeated reachability (B\"uchi) objective $\varphi_B = \square \lozenge B$ in the product MDP, computed from the MDP $\mathcal{M}$ and the obtained LBDA.  
We then exploit a new discounting and rewarding mechanism that enables the use of model-free reinforcement learning, to find an objective~policy with strong performance guarantees (i.e., probability maximization). 
{Specifically, we use Q-learning \cite{sutton2018} in this work, but other reinforcement learning methods can be applied similarly.}
Our overall approach is captured in Algorithm~\ref{alg:alg1}, and we now describe each step~in~detail.

\subsection{Design of Product MDP} \label{sub:product MDP}

Given an LTL formula $\varphi$ with atomic propositions $2^\text{AP}$, the product MDP is constructed by composing $\mathcal{M}$ with an LDBA $\mathcal{A}_\varphi$ with the alphabet $2^\text{AP}$,  which can be automatically derived from $\varphi$~\cite{hahn2015,sickert2016}.
LDBAs, similarly to deterministic Rabin automata~\cite{baier2008}, 
can be used in quantitative analysis of MDPs if they are constructed in a certain way~\cite{hahn2019}.
\begin{definition} \label{def:product MDP}
A product MDP $\mathcal{M}^\times=(S^\times,A^\times,P^\times,s_0^\times,\allowbreak B^\times)$ of an MDP $\mathcal{M}=(S,A,P,s_0,\allowbreak \text{AP}, L)$ and an LDBA $\mathcal{A}=(Q, 2^\text{AP}, \delta, q_0, B)$  is defined as: $S^\times=S\times Q$ is the set of states, $A^\times=A\cup A^\epsilon,\ A^\epsilon{:=}\{\epsilon_q|q\in Q\}$ 
is the action~set, $P^\times:S^\times\times A^\times \times S^\times \to [0,1]$ is the transition function 
\begin{align}
    &P^\times(\langle s,q \rangle,a,\langle s',q'\rangle) \notag \\ 
    &\hspace{-1pt} = \begin{cases}
        P(s,a,s') & q' = \delta(q,L(s)) \text{ and } a \notin A^\epsilon \\
        1 & a = \epsilon_{q'} \text{ and } q' \in \delta(q,\epsilon) \text{ and } s=s' \\
        0 & \text{otherwise}
    \end{cases}\hspace{-1pt}, \label{eq:prod_prob}
\end{align}
$s_0^\times$ is $\langle s_0,q_0 \rangle$, and $B^\times = \{\langle s,q \rangle \in S^\times \mid q \in B\}$ is the set of accepting states.
We say that a path $\sigma$ of the product MDP $\mathcal{M}^\times$ satisfies the B\"uchi condition $\varphi_B$ if $inf(\sigma) \cap B^\times \neq \emptyset$.
\end{definition}

\begin{algorithm}[!t]
\begin{algorithmic}
\STATE \textbf{Input:} LTL formula $\varphi$, MDP $\mathcal{M}$
\STATE Translate $\varphi$ to an LDBA $\mathcal{A}_\varphi$
\STATE Construct the product $\mathcal{M}^\times$ of $\mathcal{M}$ and $\mathcal{A}_\varphi$ %
\STATE Initialize $Q(\langle s,q \rangle, a)$ on $\mathcal{M}^\times$ %

\FOR {$t = 0,1,\dots, T$}
\STATE Derive a policy $\pi$ from $Q$
\STATE Take the action $a_t \leftarrow \pi(\langle s,q \rangle_t)$
\STATE Observe the next state $\langle s,q \rangle_{t+1}$
\STATE $Q(\langle s,q \rangle_t,a_t) \leftarrow (1-\alpha)\cdot Q(\langle s,q \rangle_t,a_t) + \alpha \cdot R_B(\langle s,q \rangle_t)$
\STATE \hspace{2em} $+ \alpha \cdot \Gamma_B(\langle s,q \rangle_t) \cdot \max_{a'} Q(\langle s,q \rangle_{t+1},a')$
\ENDFOR
\STATE Get a greedy policy $\pi^\varphi_B$ from Q
\RETURN $\pi^\varphi_B$ and $\mathcal{A}_\varphi$
\end{algorithmic}
\caption{Model-free RL-based synthesis on MDPs that maximizes the satisfaction probability of  LTL specifications.}\label{alg:alg1}
\end{algorithm}

The nondeterministic $\epsilon$-moves in the LDBA are represented by $\epsilon$-actions in the product MDP. When an $\epsilon$-action is taken, only the state of the LDBA is updated according to the corresponding $\epsilon$-move. When an MDP action is taken, the next MDP state will be determined by the transition probabilities and the LDBA makes a transition by consuming the label of the current MDP state.  Intuitively, an $\epsilon$-action can be considered as guessing the possible paths generated in the future. If, as part of iterative learning, the guess is wrong the agent cannot change its guess; however, in the next RL episode, the agent can make the correct one.

An example product MDP is illustrated in Fig.~\ref{fig:product_mdp_construction}. 
In the MDP (Fig.~\ref{fig:mdp}), 
states $s_0$ and $s_1$ are labeled by %
atomic~propositions %
${a}$ and ${b}$, respectively.
In the LDBA (Fig.~\ref{fig:ldba}), 
for simplicity, the transitions %
are labeled by Boolean formulas of the atomic propositions of $a$ and $b$ or an $\epsilon$ label, with $1$ standing for ``true''; this is equivalent to labeling the transitions using sets of atomic propositions, as in Def.~\ref{def:ldba}.~A~transition labeled by a Boolean formula is triggered upon receiving a set of atomic propositions satisfying that formula, and~a~transition labeled by an $\epsilon$ label can be (but does not have to be) triggered automatically. 
The product MDP is shown in Fig.~\ref{fig:product_mdp}. 

To distinguish the two $\epsilon$ transitions from $q_0$ to $q_1$ and from $q_0$ to $q_2$ in~Fig.~\ref{fig:ldba}, we denote them by $\epsilon_1$ and $\epsilon_2$ in Fig.~\ref{fig:product_mdp}, respectively. %
Notice that choosing~$\epsilon_2$ before choosing $\beta$ does not satisfy the B\"uchi condition although the generated paths by this policy satisfy the LTL formula. Yet, this does not constitute a problem because in such cases, there always exists a corresponding policy that generates the same paths and satisfies the B\"uchi condition (e.g. choosing $\epsilon_2$ after $\beta$).

Now, the satisfaction of the LTL objective $\varphi$ on the original MDP $\mathcal{M}$ is related to the satisfaction of the B\"uchi objective $\varphi_B$ on the product MDP $\mathcal{M}^\times$, as formalized below.

\begin{lemma}
\label{lem:product MDP}
    A memoryless policy $\pi^\varphi_B$ that maximizes the satisfaction probability of $\varphi_B$ on $\mathcal{M}^\times$ induces a finite-memory policy $\pi^\varphi$ that maximizes the satisfaction probability of $\varphi$ on $\mathcal{M}$ in Problem~\ref{problem}. 
    \vspace{-2pt}
\end{lemma}
\begin{proof} 
Follows from the proof of Theorem~3 in~\cite{sickert2016}.
\end{proof}
\vspace{-2pt}

Therefore, the behavior of the induced policy $\pi^\varphi$ can be described by the policy $\pi^\varphi_B$ and the LDBA $\mathcal{A}_\varphi$ derived directly from the LTL formula $\varphi$. Initially, $\mathcal{A}_\varphi$  is reset to its start state $q_0$ and whenever the MDP $\mathcal{M}$ makes a transition from $s$ to $s'$, $\mathcal{A}_\varphi$  updates its current state from $q$ to $\delta(q,L(s))$. The action to be selected in an MDP state $s$ when $\mathcal{A}_\varphi$ is in a state $q$ is determined by  $\pi^\varphi_B$ as follows: if $\pi^\varphi_B(\langle s,q \rangle)$ is an $\epsilon$-action $\epsilon_{q'}$, $\mathcal{A}_\varphi$ changes its state to $q'$ and the action $\pi^\varphi_B(\langle s,q' \rangle)$ is selected; otherwise, $\pi^\varphi_B(\langle s,q \rangle)$ is selected.

\subsection{Learning for B\"uchi Conditions with Discounted Rewards} 

Our goal is to learn a policy that maximizes the probability of satisfying a given B\"uchi objective. 
By Lemma~\ref{lem:product MDP}, in what follows, we assume policies are memoryless since they are sufficient for B\"uchi objectives.
For simplicity, we omit the superscript $ ^\times$ and write $\mathcal{M}=(S,A,P,s_0,\allowbreak B)$ and $s\in S$ instead of $\mathcal{M}^\times=(S^\times,A^\times,P^\times,s_0^\times,\allowbreak B^\times)$ and $\langle s,q \rangle \in S^\times$.

We propose a model-free learning method that 
uses carefully crafted rewards and state-dependent discounting based on the B\"uchi condition such that an optimal policy $\pi^*$ maximizing the expected return is also an objective policy $\pi^\varphi_B$  maximizing the satisfaction probabilities.
Specifically, 
we define the return of a path as a function of these rewards and discount factors in such a way that the value of a state, the expected return from that state, approaches the probability of satisfying the objective as the discount factor $\gamma$ goes to~1. %

\begin{theorem} \label{theorem:buchi}
For a given MDP $\mathcal{M}$ with $B \subseteq S$, 
{the value function $v_\pi^\gamma$ for the policy~$\pi$ and the discount factor $\gamma$ satisfies}
\vspace{-10pt}
\begin{align} \label{eq:thm1}
    \lim_{\gamma \to 1^{{-}}} v_\pi^\gamma(s) = Pr_\pi(s \models \square \lozenge B)
\end{align} 
for all states $s \in S$, if the return of a path is defined~as
\begin{align}
    G_t(\sigma) &:= \sum\nolimits_{i=0}^{\infty} R_B(\sigma[t{+}i])\cdot \prod\nolimits_{j=0}^{i-1} \Gamma_B(\sigma[t{+}j]) \label{eq:return_updated}
\end{align} 
where $\prod_{j=0}^{-1} := 1$, $R_B:S\to[0,1)$ and $\Gamma_B:S\to(0,1)$ are the reward and the discount functions defined as:
\begin{align}
    \hspace{-8pt}R_B(s) := \begin{cases}
        1-\gamma_B &  \hspace{-2pt}s \in B \\
        0 & \hspace{-2pt}s \notin B
    \end{cases}, \
    \Gamma_B(s) := \begin{cases}
        \gamma_B & \hspace{-2pt}s \in B \\
        \gamma & \hspace{-2pt}s \notin B
    \end{cases} \label{eq:reward_discount}
\end{align}
Here, we set $\gamma_B = \gamma_B (\gamma)$ as a function of $\gamma$ such that
\begin{align} \label{eq:gamma_lim}
    \lim_{\gamma \to 1^-} \frac{1 - \gamma}{1 - \gamma_B (\gamma)} = 0. 
\end{align}
\end{theorem}

Before proving Theorem~\ref{theorem:buchi}, we develop bounds on $G_t(\sigma)$.  %

\begin{lemma} \label{lemma:inequalities}
For all paths and $G_t(\sigma)$ from~\eqref{eq:return_updated}, it holds that
\begin{align}
    0 \leq \gamma G_{t+1}(\sigma) \leq G_t(\sigma) \leq 1-\gamma_B + \gamma_B G_{t+1}(\sigma) \leq 1 \label{eq:ineqs}
\end{align}
\end{lemma}
\begin{proof}
Since there is no negative reward, $G_t \geq 0$ holds. By the return definition, replacing $\gamma$ with 1 yields a larger or equal return, which constitutes the following upper bound on the return: $G_t(\sigma)\leq 1-\gamma_B^{b} \leq 1$, where $b$ is the number of $B$ states visited. Return $G_t(\sigma)$ from~\eqref{eq:return_updated} satisfies
\vspace{-2pt}
\begin{align}
    G_t(\sigma) = \begin{cases}
        1 + \gamma_B (G_{t+1}(\sigma)-1) & \sigma[t] \in B \\
        \gamma G_{t+1}(\sigma) & \sigma[t] \notin B
    \end{cases} \label{eq:return_recursive}
\end{align}
From $G_t(\sigma)\leq 1$ it follows that $1 + \gamma_B (G_{t+1}(\sigma)-1)\geq \gamma G_{t+1}(\sigma)$, which with (\ref{eq:return_recursive}) proves the other inequalities.
\end{proof}

Lemma \ref{lemma:inequalities} implies that replacing a prefix of a path with states belonging to $B$ never decreases the return of a path and similarly replacing with states that do not belong to $B$ never increases the return. The result is particularly useful~when we establish upper and lower bounds on the value of a state.
\vspace{-5pt}

The next lemma shows that under a policy, the values of states in the accepting BSCCs of the induced Markov chain approach 1 in the limit; thus, is the key to proving Theorem~\ref{theorem:buchi}. 
\begin{lemma} \label{lemma:v1}
Let $\text{BSCC}(\mathcal{M}_\pi)$ denote the set of all BSCCs of an induced Markov chain $\mathcal{M}_\pi$ and let $B_\pi$ denote the set of $B$ states that belong to a BSCC of $\mathcal{M}_\pi$ -- i.e., 
 \begin{align} 
     B_\pi := \{s \mid s \in B, s\in T, T \in \text{BSCC}(\mathcal{M}_\pi) \}.
 \label{eq:B_pi}
 \end{align}
Then, for any state $s$ in~$B_\pi$ %
\begin{align} \label{eq:v1}
    \lim_{\gamma \to 1^{{-}}} v_\pi^\gamma(s) = 1.
\end{align}
\end{lemma}

\begin{proof}
For any fixed $t \in \mathbb{N}$, let $N_t$ be the stopping time of first returning to the state $s \in S$ after leaving it at $t$, 
\begin{align}
N_t = \min \{\tau \mid \sigma[t{+}\tau]=s,\tau>0 \}. 
\end{align}
Then by~\eqref{eq:expected return}, it holds that
\begin{align}
    v_\pi^\gamma(s) &= 1-\gamma_B + \gamma_B\mathbb{E}_{\pi}[G_{t+1}(\sigma) \mid \sigma[t]{=}s] \notag \\
    &= 1-\gamma_B+\gamma_B\mathbb{E}_{\pi}\Big[G_{t+1:t+N_t-1}(\sigma) \notag \\ &+\left(\prod\nolimits_{i=1}^{N_t-1} \Gamma({\sigma[t{+}i]})\right)\cdot G_{t+N_t}(\sigma) \mid \sigma[t]{=}s\Big],
\end{align}
since once a state $s\in B_\pi$ is visited, almost surely it~is~visited again \cite{baier2008}. 
Using that $G_t(\sigma) \geq \gamma G_{t+1}(\sigma)$, we obtain
\begin{align}
    v_\pi^\gamma(s) &\geq 1-\gamma_B + \gamma_B\mathbb{E}_{\pi}\left[\gamma^{N_t-1} G_{t+N_t}(\sigma) \mid \sigma[t]{=}s\right]\notag\\
    &\overset{\text{\small \ding{192}}}{\geq} 1-\gamma_B + \gamma_B\mathbb{E}_{\pi}\left[\gamma^{N_t-1} \mid \sigma[t]{=}s\right]v_\pi(s)\notag\\
    &\overset{\text{\small \ding{193}}}{\geq}  1-\gamma_B + \gamma_B \gamma^{\mathbb{E}_{\pi}\left[N_t-1 \mid \sigma[t]{=}s\right]}v_\pi(s)\notag\\
    &\geq 1-\gamma_B + \gamma_B \gamma^n v_\pi(s) \label{eq:lb_1v}
\end{align}
where \ding{192} holds by the Markov property, \ding{193} holds by the Jensen's inequality, and $n\geq1$ is a constant.
From~\eqref{eq:lb_1v}, %
\begin{align}
    v_\pi^\gamma(s) &\geq \frac{1-\gamma_B}{1-\gamma_B \gamma^n} \geq \frac{1-\gamma_B}{1-\gamma_B (1-n(1-\gamma))} \notag \\
    &= \frac{1}{1 + n\frac{1-\gamma}{1-\gamma_B} - n(1-\gamma)}. \label{eq:lb_1} 
\end{align}
where the second ``$\geq$'' holds by $(1- (1-\gamma))^n \geq 1 - n (1-\gamma)$ for $\gamma \in (0,1)$.
Finally, since $v_\pi^\gamma(s) \leq 1$ by Lemma~\ref{lemma:inequalities}, letting $\gamma, \gamma_B \to 1^-$ under the condition~\eqref{eq:gamma_lim} results in~\eqref{eq:v1}.
\end{proof}

We now prove Theorem \ref{theorem:buchi}. 

\begin{proof}[Proof of Theorem~\ref{theorem:buchi}]
First, we divide the expected return of a random path $\sigma$ from a state $s \in S$ by whether it visits the states $B \subseteq S$ infinitely often:
\begin{align}
    v_\pi^\gamma(s) &= \mathbb{E}_{\pi}[G_t(\sigma) \mid \sigma[t]{=}s, \sigma \models \square \lozenge B]Pr_\pi(s\models \square \lozenge B) \notag \\
    &\hspace{-12pt}+ \mathbb{E}_{\pi}[G_t(\sigma) \mid \sigma[t]{=}s, \sigma \not\models \square \lozenge B]Pr_\pi(s \not\models \square\lozenge B) \label{eq:buchi_v_two_parts}
\end{align}
for some fixed $t \in \mathbb{N}$.
let $M_t$ be the stopping time of first reaching a state in $B_\pi$ after leaving $s$ at $t$,
\begin{align}
M_t =\min \{\tau \mid \sigma[t{+}\tau]\in B_\pi,\tau>0 \}    
\end{align}
where $B_\pi$ is defined as in (\ref{eq:B_pi}).
Then, it holds that
\begin{align}
    & \mathbb{E}_{\pi}[G_t(\sigma) \mid \sigma[t]{=}s, \sigma \models \square \lozenge B] \\
    & \qquad  \overset{\text{\small \ding{192}}}{=} \mathbb{E}_{\pi}[G_t(\sigma) \mid \sigma[t]{=}s, \sigma \models \lozenge B_\pi]\notag \\
    & \qquad \overset{\text{\small \ding{193}}}{\geq} \mathbb{E}_{\pi}\left[\gamma^{M_t} G_{t+M_t}(\sigma) \mid \sigma[t]{=}s, \sigma \models \lozenge B_\pi\right] \notag\\
    & \qquad \overset{\text{\small \ding{194}}}{\geq}\mathbb{E}_{\pi}\left[\gamma^{M_t} \mid \sigma[t]{=}s, \sigma \models \lozenge B_\pi\right]v_{\pi,\text{min}}^\gamma(B_\pi)  \notag\\
    & \qquad \overset{\text{\small \ding{195}}}{\geq}\gamma^{\mathbb{E}_{\pi}\left[M_t \mid \sigma[t]{=}s, \sigma \models \lozenge B_\pi\right]}v_{\pi,\text{min}}^\gamma(B_\pi)  \notag\\
    & \qquad = \gamma^m v_{\pi,\text{min}}^\gamma(B_\pi),
\end{align}
where $v_{\pi,\text{min}}^\gamma(B_\pi) = \min_{s\in B_\pi} v_\pi^\gamma(s)$ and m is constant.  
Here, \ding{192} holds because a path $\sigma \models \square \lozenge B$ almost surely eventually enters an accepting BSCC, it eventually reaches a state $s\in B_\pi$ almost surely, \ding{193}, \ding{194} and \ding{195} hold due to Lemma~\ref{lemma:inequalities}, the Markov property and Jensen's inequality. %
From~(\ref{eq:buchi_v_two_parts}), we have
\begin{align}
    v_\pi^\gamma(s) &\geq \gamma^m v_\pi(B_\pi) Pr_\pi(s\models \square \lozenge B). \label{eq:buchi_v_lower}
\end{align}
Similarly, let $M_t'$ be the stopping time of first reaching a rejecting BSCC of $\mathcal{M}_\pi$ after leaving $s$ at $t$. Then
\begin{align}
    M_\pi' & =\min \big\{\tau \mid \sigma[t{+}\tau]\in T, T \cap B = \varnothing, \notag \\
    &\hspace{5em} T \in BSCC(\mathcal{M}_\pi), \tau>0 \big\}
\end{align}
 denoting the number of time steps before a rejecting BSCC 
 
 \noindent is reached. Thus, from Lemma~\ref{lemma:inequalities} and the Markov~property
\begin{align}
    \mathbb{E}_{\pi}[G_t(\sigma) \mid \sigma[t]{=}s, \sigma \not\models \square \lozenge B]  & \notag \\
    &\hspace{-10em}\leq \mathbb{E}_{\pi}\left[1-\gamma_B^{M_\pi'}\mid \sigma[t]{=}s, \sigma \not\models \square \lozenge B\right] \notag\\
    &\hspace{-10em}\leq 1-\gamma_B^{\mathbb{E}_{\pi}\left[M_\pi' \mid \sigma[t]{=}s, \sigma \not\models \square \lozenge B\right]}= 1-\gamma_B^{m'}
\end{align}
where $m'$ is also constant. From this upper bound and (\ref{eq:buchi_v_two_parts}) %
\begin{align*}
    v_\pi^\gamma(s) &\leq  Pr_\pi(s\models \square \lozenge B) + (1-\gamma_B^{m'})Pr_\pi(s\not\models \square\lozenge B). %
\end{align*}
Both the above upper bound and the lower bound from (\ref{eq:buchi_v_lower}) %
go to the probability of satisfying the formula 
as $\gamma$ approaches 1 from below, thus concluding the proof. 
\end{proof}

Theorem \ref{theorem:buchi} suggests that the limit of the optimal state~values is equal to the maximal probabilities as $\gamma$ goes to~1; this is captured by the next corollary whose proof follows from the definition of the optimal policies and maximal~probabilities.
\begin{corollary} \label{corollary:buchi_optimal}
For all states $s\in S$ the following holds:
\begin{align} \label{eq:coro1}
    \lim_{\gamma \to 1^{{-}}} v_*^\gamma(s) = Pr_{\text{max}}(s \models \square \lozenge B).
\end{align}
\end{corollary}

\begin{remark}
    From Theorem 1 of~\cite{jaakkola1994},  $\gamma < 1$ ensures convergence of the model-free learning to the unique solution. With $\gamma = 1$, the result may converge to a non-optimal policy~\cite{brazdil2014} {as there might exist multiple fixed-point solutions}. 
\end{remark}

{Finally, as the policies are discrete, the convergence of~\eqref{eq:thm1} and~\eqref{eq:coro1} is achieved after some threshold $\gamma'$, as stated below.}

\begin{corollary}
There exists a $\gamma'$ such that for all $\gamma>\gamma'$ and for all states $s \in S$, the optimal policy $\pi^*$ satisfies
\begin{align}
     Pr_{\pi^*}\left(s \models \square \lozenge B \right) = Pr_\text{max}(s \models \square \lozenge B).
\end{align}
\end{corollary}
\begin{proof}
Let $d_\text{min}$ be the minimum positive difference between the satisfaction probabilities of two policies:
\begin{align*}
    d_\text{min} &:= \min\big\{|Pr_{\pi_1}(s \models \square \lozenge B)-Pr_{\pi_2}(s \models \square \lozenge B)| \notag \\
    & \quad \mid s\in S, Pr_{\pi_1}(s \models \square \lozenge B)\neq Pr_{\pi_2}(s \models \square \lozenge B)\big\}
\end{align*}
and let $\gamma'$ be the discount factor such that
\begin{align}
    \max \big\{ |v_\pi^\gamma(s) - Pr_\pi(s \models \square \lozenge B)| \mid s\in S \big\} < d_\text{min}/2.
\end{align}
Now, suppose a policy $\pi'$ that maximizes the satisfaction probability is not optimal for $\gamma'$, then the optimal value of all states must be larger than $Pr_\text{max}(s \models \square \lozenge B)-d_\text{min}/2$, which is not possible due to the definition of $d_\text{min}$.
\end{proof}

%% file: graph1.tex
\begin{tikzpicture}[>=latex',line join=bevel,]
  \pgfsetlinewidth{1bp}
\pgfsetcolor{black}
  \draw [->] (9.9836bp,85.29bp) .. controls (8.0505bp,95.389bp) and (10.723bp,105.0bp)  .. (18.0bp,105.0bp) .. controls (22.662bp,105.0bp) and (25.434bp,101.06bp)  .. (26.016bp,85.29bp);
  \definecolor{strokecol}{rgb}{0.0,0.0,0.0};
  \pgfsetstrokecolor{strokecol}
  \draw (18.0bp,112.5bp) node {\LARGE $1$};
  \draw [->] (34.285bp,77.143bp) .. controls (48.148bp,84.074bp) and (68.398bp,94.199bp)  .. (94.31bp,107.16bp);
  \draw (64.0bp,102.5bp) node {\LARGE $\epsilon$};
  \draw [->] (34.285bp,61.027bp) .. controls (48.058bp,54.284bp) and (68.137bp,44.454bp)  .. (93.967bp,31.808bp);
  \draw (64.0bp,58.5bp) node {\LARGE $\epsilon$};
  \draw [->] (104.95bp,137.17bp) .. controls (103.56bp,147.6bp) and (106.58bp,157.0bp)  .. (114.0bp,157.0bp) .. controls (118.87bp,157.0bp) and (121.84bp,152.95bp)  .. (123.05bp,137.17bp);
  \draw (114.0bp,164.5bp) node {\LARGE $a$};
  \draw [->] (134.05bp,107.75bp) .. controls (150.59bp,100.11bp) and (174.16bp,89.235bp)  .. (201.43bp,76.646bp);
  \draw (168.0bp,104.5bp) node {\LARGE $\neg a$};
  \draw [->] (104.95bp,42.17bp) .. controls (103.56bp,52.599bp) and (106.58bp,62.0bp)  .. (114.0bp,62.0bp) .. controls (118.87bp,62.0bp) and (121.84bp,57.951bp)  .. (123.05bp,42.17bp);
  \draw (114.0bp,69.5bp) node {\LARGE $b$};
  \draw [->] (134.05bp,31.062bp) .. controls (150.59bp,38.534bp) and (174.16bp,49.187bp)  .. (201.43bp,61.514bp);
  \draw (168.0bp,59.5bp) node {\LARGE $\neg b$};
  \draw [->] (209.37bp,84.916bp) .. controls (207.11bp,95.15bp) and (209.99bp,105.0bp)  .. (218.0bp,105.0bp) .. controls (223.26bp,105.0bp) and (226.3bp,100.76bp)  .. (226.63bp,84.916bp);
  \draw (218.0bp,112.5bp) node {\LARGE $1$};
\begin{scope}
  \definecolor{strokecol}{rgb}{0.0,0.0,0.0};
  \pgfsetstrokecolor{strokecol}
  \draw (18.0bp,69.0bp) ellipse (18.0bp and 18.0bp);
  \draw (18.0bp,69.0bp) node {\LARGE $q_0$};
\end{scope}
\begin{scope}
  \definecolor{strokecol}{rgb}{0.0,0.0,0.0};
  \pgfsetstrokecolor{strokecol}
  \draw (114.0bp,117.0bp) ellipse (18.0bp and 18.0bp);
  \draw (114.0bp,117.0bp) ellipse (22.0bp and 22.0bp);
  \draw (114.0bp,117.0bp) node {\LARGE $q_1$};
\end{scope}
\begin{scope}
  \definecolor{strokecol}{rgb}{0.0,0.0,0.0};
  \pgfsetstrokecolor{strokecol}
  \draw (114.0bp,22.0bp) ellipse (18.0bp and 18.0bp);
  \draw (114.0bp,22.0bp) ellipse (22.0bp and 22.0bp);
  \draw (114.0bp,22.0bp) node {\LARGE $q_2$};
\end{scope}
\begin{scope}
  \definecolor{strokecol}{rgb}{0.0,0.0,0.0};
  \pgfsetstrokecolor{strokecol}
  \draw (218.0bp,69.0bp) ellipse (18.0bp and 18.0bp);
  \draw (218.0bp,69.0bp) node {\LARGE $q_3$};
\end{scope}
\end{tikzpicture}

%% file: graph2.tex
\begin{tikzpicture}[>=latex',line join=bevel,]
  \pgfsetlinewidth{1bp}
\pgfsetcolor{black}
  \draw [->] (34.0bp,52.832bp) .. controls (40.074bp,55.718bp) and (47.176bp,58.552bp)  .. (54.0bp,60.0bp) .. controls (67.478bp,62.86bp) and (71.378bp,62.065bp)  .. (85.0bp,60.0bp) .. controls (87.645bp,59.599bp) and (90.381bp,59.003bp)  .. (102.89bp,55.309bp);
  \draw [->] (35.163bp,38.1bp) .. controls (51.42bp,32.512bp) and (75.861bp,24.11bp)  .. (102.97bp,14.793bp);
  \draw [->] (102.93bp,46.215bp) .. controls (97.618bp,44.17bp) and (91.093bp,42.021bp)  .. (85.0bp,41.0bp) .. controls (72.299bp,38.872bp) and (58.067bp,39.142bp)  .. (36.04bp,41.154bp);
  \definecolor{strokecol}{rgb}{0.0,0.0,0.0};
  \pgfsetstrokecolor{strokecol}
  \draw (69.5bp,48.5bp) node {\LARGE $0.9$};
  \draw [->] (125.01bp,48.706bp) .. controls (138.89bp,45.814bp) and (163.24bp,40.742bp)  .. (192.29bp,34.689bp);
  \draw (158.5bp,51.5bp) node {\LARGE $0.1$};
  \draw [->] (125.17bp,11.853bp) .. controls (137.31bp,12.927bp) and (157.21bp,15.119bp)  .. (174.0bp,19.0bp) .. controls (177.05bp,19.704bp) and (180.2bp,20.56bp)  .. (192.98bp,24.591bp);
  \draw (158.5bp,26.5bp) node {\LARGE $1.0$};
  \draw [->] (228.31bp,31.607bp) .. controls (233.98bp,31.769bp) and (240.25bp,31.92bp)  .. (246.0bp,32.0bp) .. controls (259.78bp,32.191bp) and (263.22bp,32.237bp)  .. (277.0bp,32.0bp) .. controls (279.51bp,31.957bp) and (282.14bp,31.893bp)  .. (294.99bp,31.469bp);
  \draw [->] (294.85bp,22.049bp) .. controls (289.73bp,18.493bp) and (283.38bp,14.82bp)  .. (277.0bp,13.0bp) .. controls (263.75bp,9.2223bp) and (259.41bp,9.8345bp)  .. (246.0bp,13.0bp) .. controls (242.24bp,13.887bp) and (238.43bp,15.23bp)  .. (225.52bp,21.264bp);
  \draw (261.5bp,20.5bp) node {\LARGE $1.0$};
\begin{scope}
  \definecolor{strokecol}{rgb}{0.0,0.0,0.0};
  \pgfsetstrokecolor{strokecol}
  \draw (18.0bp,44.0bp) ellipse (18.0bp and 18.0bp);
  \draw (18.0bp,44.0bp) node {\begin{tabular}{c} \LARGE$s_0$ \\ \Large\{{\LARGE$a$}\}  \end{tabular}};
\end{scope}
\begin{scope}
  \definecolor{strokecol}{rgb}{0.0,0.0,0.0};
  \pgfsetstrokecolor{strokecol}
  \draw (125.0bp,62.0bp) -- (103.0bp,62.0bp) -- (103.0bp,40.0bp) -- (125.0bp,40.0bp) -- cycle;
  \draw (114.0bp,51.0bp) node {\LARGE $\alpha$};
\end{scope}
\begin{scope}
  \definecolor{strokecol}{rgb}{0.0,0.0,0.0};
  \pgfsetstrokecolor{strokecol}
  \draw (125.0bp,22.0bp) -- (103.0bp,22.0bp) -- (103.0bp,0.0bp) -- (125.0bp,0.0bp) -- cycle;
  \draw (114.0bp,11.0bp) node {\LARGE $\beta$};
\end{scope}
\begin{scope}
  \definecolor{strokecol}{rgb}{0.0,0.0,0.0};
  \pgfsetstrokecolor{strokecol}
  \draw (210.0bp,31.0bp) ellipse (18.0bp and 18.0bp);
  \draw (210.0bp,31.0bp) node {\begin{tabular}{c} \LARGE$s_1$ \\ \Large\{{\LARGE$b$}\}   \end{tabular}};
\end{scope}
\begin{scope}
  \definecolor{strokecol}{rgb}{0.0,0.0,0.0};
  \pgfsetstrokecolor{strokecol}
  \draw (317.0bp,42.0bp) -- (295.0bp,42.0bp) -- (295.0bp,20.0bp) -- (317.0bp,20.0bp) -- cycle;
  \draw (306.0bp,31.0bp) node {\LARGE $\theta$};
\end{scope}
\end{tikzpicture}

%% file: graph3.tex
\begin{tikzpicture}[>=latex',line join=bevel,]
  \pgfsetlinewidth{1bp}
\pgfsetcolor{black}
  \draw [->] (26.145bp,192.45bp) .. controls (32.329bp,203.29bp) and (41.841bp,216.91bp)  .. (54.0bp,224.93bp) .. controls (65.484bp,232.51bp) and (80.575bp,236.0bp)  .. (102.92bp,238.61bp);
  \draw [->] (36.057bp,179.5bp) .. controls (52.186bp,182.69bp) and (75.824bp,187.37bp)  .. (102.91bp,192.73bp);
  \draw [->] (35.608bp,171.34bp) .. controls (51.835bp,167.12bp) and (75.934bp,160.84bp)  .. (102.8bp,153.84bp);
  \draw [->] (32.58bp,165.15bp) .. controls (49.168bp,152.88bp) and (76.328bp,132.79bp)  .. (102.88bp,113.15bp);
  \draw [->] (106.12bp,227.72bp) .. controls (100.79bp,220.75bp) and (93.226bp,211.99bp)  .. (85.0bp,205.93bp) .. controls (72.94bp,197.04bp) and (57.826bp,189.98bp)  .. (35.32bp,181.36bp);
  \definecolor{strokecol}{rgb}{0.0,0.0,0.0};
  \pgfsetstrokecolor{strokecol}
  \draw (69.5bp,213.43bp) node {\LARGE $0.9$};
  \draw [->] (125.08bp,238.93bp) .. controls (139.57bp,238.93bp) and (165.45bp,238.93bp)  .. (195.76bp,238.93bp);
  \draw (158.5bp,246.43bp) node {\LARGE $0.1$};
  \draw [->] (125.13bp,199.38bp) .. controls (137.23bp,204.26bp) and (157.09bp,212.4bp)  .. (174.0bp,219.93bp) .. controls (178.67bp,222.01bp) and (183.61bp,224.29bp)  .. (197.74bp,230.98bp);
  \draw (158.5bp,227.43bp) node {\LARGE $1.0$};
  \draw [->] (125.08bp,150.93bp) .. controls (138.71bp,150.93bp) and (162.41bp,150.93bp)  .. (191.95bp,150.93bp);
  \draw (158.5bp,158.43bp) node {\LARGE $1.0$};
  \draw [->] (125.08bp,102.93bp) .. controls (138.79bp,100.47bp) and (162.71bp,96.161bp)  .. (192.34bp,90.827bp);
  \draw (158.5bp,107.43bp) node {\LARGE $1.0$};
  \draw [->] (230.31bp,247.09bp) .. controls (237.4bp,250.2bp) and (245.92bp,253.37bp)  .. (254.0bp,254.93bp) .. controls (266.76bp,257.39bp) and (281.27bp,256.28bp)  .. (302.84bp,252.66bp);
  \draw [->] (231.42bp,233.18bp) .. controls (248.63bp,227.5bp) and (274.97bp,218.81bp)  .. (302.96bp,209.57bp);
  \draw [->] (227.85bp,250.66bp) .. controls (235.29bp,256.62bp) and (244.81bp,263.67bp)  .. (254.0bp,268.93bp) .. controls (266.58bp,276.12bp) and (281.63bp,282.38bp)  .. (303.0bp,290.28bp);
  \draw [->] (302.87bp,243.1bp) .. controls (297.64bp,240.28bp) and (291.21bp,237.35bp)  .. (285.0bp,235.93bp) .. controls (270.96bp,232.72bp) and (254.95bp,233.06bp)  .. (231.74bp,235.6bp);
  \draw (269.5bp,243.43bp) node {\LARGE $1.0$};
  \draw [->] (325.1bp,201.81bp) .. controls (337.18bp,197.29bp) and (357.04bp,189.78bp)  .. (374.0bp,182.93bp) .. controls (377.32bp,181.59bp) and (380.78bp,180.16bp)  .. (393.71bp,174.72bp);
  \draw (358.5bp,201.43bp) node {\LARGE $1.0$};
  \draw [->] (325.08bp,293.93bp) .. controls (338.71bp,293.93bp) and (362.41bp,293.93bp)  .. (391.95bp,293.93bp);
  \draw (358.5bp,301.43bp) node {\LARGE $1.0$};
  \draw [->] (233.9bp,160.61bp) .. controls (240.18bp,163.2bp) and (247.24bp,165.63bp)  .. (254.0bp,166.93bp) .. controls (266.71bp,169.38bp) and (281.22bp,169.21bp)  .. (302.82bp,167.48bp);
  \draw [->] (235.55bp,144.84bp) .. controls (241.51bp,143.2bp) and (248.0bp,141.45bp)  .. (254.0bp,139.93bp) .. controls (266.88bp,136.66bp) and (281.4bp,133.26bp)  .. (302.91bp,128.38bp);
  \draw [->] (302.85bp,156.98bp) .. controls (297.73bp,153.42bp) and (291.38bp,149.75bp)  .. (285.0bp,147.93bp) .. controls (272.49bp,144.36bp) and (258.23bp,144.24bp)  .. (235.86bp,146.49bp);
  \draw (269.5bp,155.43bp) node {\LARGE $0.9$};
  \draw [->] (325.08bp,165.93bp) .. controls (338.71bp,165.93bp) and (362.41bp,165.93bp)  .. (391.95bp,165.93bp);
  \draw (358.5bp,173.43bp) node {\LARGE $0.1$};
  \draw [->] (325.19bp,129.45bp) .. controls (337.34bp,133.36bp) and (357.26bp,140.05bp)  .. (374.0bp,146.93bp) .. controls (377.55bp,148.39bp) and (381.24bp,150.0bp)  .. (394.25bp,156.03bp);
  \draw (358.5bp,154.43bp) node {\LARGE $1.0$};
  \draw [->] (436.19bp,164.38bp) .. controls (453.02bp,163.2bp) and (476.02bp,161.59bp)  .. (502.98bp,159.7bp);
  \draw [->] (525.18bp,155.3bp) .. controls (538.04bp,150.71bp) and (559.28bp,141.81bp)  .. (574.0bp,128.93bp) .. controls (582.88bp,121.16bp) and (590.58bp,110.72bp)  .. (601.74bp,92.396bp);
  \draw (558.5bp,155.43bp) node {\LARGE $1.0$};
  \draw [->] (234.72bp,78.848bp) .. controls (251.87bp,72.16bp) and (276.12bp,62.704bp)  .. (302.91bp,52.256bp);
  \draw [->] (236.26bp,87.391bp) .. controls (263.69bp,87.961bp) and (309.98bp,88.922bp)  .. (347.35bp,89.697bp);
  \draw [->] (325.08bp,48.262bp) .. controls (339.57bp,48.696bp) and (365.45bp,49.473bp)  .. (395.76bp,50.382bp);
  \draw (358.5bp,57.429bp) node {\LARGE $0.9$};
  \draw [->] (325.04bp,41.486bp) .. controls (352.0bp,26.524bp) and (423.47bp,-8.5271bp)  .. (485.0bp,1.929bp) .. controls (527.75bp,9.1949bp) and (538.76bp,16.649bp)  .. (574.0bp,41.929bp) .. controls (579.5bp,45.875bp) and (585.02bp,50.611bp)  .. (597.41bp,62.429bp);
  \draw (469.5bp,9.429bp) node {\LARGE $0.1$};
  \draw [->] (369.5bp,92.261bp) .. controls (401.64bp,98.634bp) and (497.91bp,114.6bp)  .. (574.0bp,94.929bp) .. controls (577.97bp,93.904bp) and (581.98bp,92.368bp)  .. (594.88bp,85.923bp);
  \draw (469.5bp,112.43bp) node {\LARGE $1.0$};
  \draw [->] (436.07bp,294.6bp) .. controls (450.08bp,294.93bp) and (468.61bp,295.21bp)  .. (485.0bp,294.93bp) .. controls (487.51bp,294.89bp) and (490.14bp,294.82bp)  .. (502.99bp,294.4bp);
  \draw [->] (502.85bp,284.98bp) .. controls (497.73bp,281.42bp) and (491.38bp,277.75bp)  .. (485.0bp,275.93bp) .. controls (471.75bp,272.15bp) and (467.47bp,273.04bp)  .. (454.0bp,275.93bp) .. controls (450.28bp,276.73bp) and (446.49bp,277.9bp)  .. (433.37bp,283.23bp);
  \draw (469.5bp,283.43bp) node {\LARGE $1.0$};
  \draw [->] (432.15bp,49.023bp) .. controls (446.53bp,47.418bp) and (467.11bp,44.91bp)  .. (485.0bp,41.929bp) .. controls (487.57bp,41.5bp) and (490.27bp,41.006bp)  .. (502.75bp,38.47bp);
  \draw [->] (431.88bp,55.399bp) .. controls (448.99bp,59.675bp) and (474.79bp,66.128bp)  .. (502.79bp,73.127bp);
  \draw [->] (502.81bp,29.581bp) .. controls (497.57bp,26.966bp) and (491.15bp,24.245bp)  .. (485.0bp,22.929bp) .. controls (471.53bp,20.047bp) and (467.14bp,18.787bp)  .. (454.0bp,22.929bp) .. controls (447.53bp,24.969bp) and (441.19bp,28.472bp)  .. (427.3bp,38.501bp);
  \draw (469.5bp,30.429bp) node {\LARGE $0.9$};
  \draw [->] (525.25bp,39.313bp) .. controls (537.45bp,43.111bp) and (557.41bp,49.695bp)  .. (574.0bp,56.929bp) .. controls (577.68bp,58.535bp) and (581.51bp,60.366bp)  .. (594.12bp,66.895bp);
  \draw (558.5bp,64.429bp) node {\LARGE $0.1$};
  \draw [->] (525.01bp,75.929bp) .. controls (538.74bp,75.929bp) and (562.73bp,75.929bp)  .. (591.94bp,75.929bp);
  \draw (558.5bp,83.429bp) node {\LARGE $1.0$};
  \draw [->] (628.31bp,76.536bp) .. controls (633.98bp,76.698bp) and (640.25bp,76.849bp)  .. (646.0bp,76.929bp) .. controls (659.78bp,77.12bp) and (663.22bp,77.166bp)  .. (677.0bp,76.929bp) .. controls (679.51bp,76.886bp) and (682.14bp,76.822bp)  .. (694.99bp,76.398bp);
  \draw [->] (694.85bp,66.978bp) .. controls (689.73bp,63.422bp) and (683.38bp,59.749bp)  .. (677.0bp,57.929bp) .. controls (663.75bp,54.151bp) and (659.41bp,54.764bp)  .. (646.0bp,57.929bp) .. controls (642.24bp,58.816bp) and (638.43bp,60.159bp)  .. (625.52bp,66.193bp);
  \draw (661.5bp,65.429bp) node {\LARGE $1.0$};
\begin{scope}
  \definecolor{strokecol}{rgb}{0.0,0.0,0.0};
  \pgfsetstrokecolor{strokecol}
  \draw (18.0bp,175.93bp) ellipse (18.0bp and 18.0bp);
  \draw (18.0bp,175.93bp) node {\LARGE $q_0{,}s_0$};
\end{scope}
\begin{scope}
  \definecolor{strokecol}{rgb}{0.0,0.0,0.0};
  \pgfsetstrokecolor{strokecol}
  \draw (125.0bp,249.93bp) -- (103.0bp,249.93bp) -- (103.0bp,227.93bp) -- (125.0bp,227.93bp) -- cycle;
  \draw (114.0bp,238.93bp) node {\LARGE $\alpha$};
\end{scope}
\begin{scope}
  \definecolor{strokecol}{rgb}{0.0,0.0,0.0};
  \pgfsetstrokecolor{strokecol}
  \draw (125.0bp,205.93bp) -- (103.0bp,205.93bp) -- (103.0bp,183.93bp) -- (125.0bp,183.93bp) -- cycle;
  \draw (114.0bp,194.93bp) node {\LARGE $\beta$};
\end{scope}
\begin{scope}
  \definecolor{strokecol}{rgb}{0.0,0.0,0.0};
  \pgfsetstrokecolor{strokecol}
  \draw (125.0bp,161.93bp) -- (103.0bp,161.93bp) -- (103.0bp,139.93bp) -- (125.0bp,139.93bp) -- cycle;
  \draw (114.0bp,150.93bp) node {\LARGE $\epsilon_1$};
\end{scope}
\begin{scope}
  \definecolor{strokecol}{rgb}{0.0,0.0,0.0};
  \pgfsetstrokecolor{strokecol}
  \draw (125.0bp,115.93bp) -- (103.0bp,115.93bp) -- (103.0bp,93.93bp) -- (125.0bp,93.93bp) -- cycle;
  \draw (114.0bp,104.93bp) node {\LARGE $\epsilon_2$};
\end{scope}
\begin{scope}
  \definecolor{strokecol}{rgb}{0.0,0.0,0.0};
  \pgfsetstrokecolor{strokecol}
  \draw (214.0bp,238.93bp) ellipse (18.0bp and 18.0bp);
  \draw (214.0bp,238.93bp) node {\LARGE $q_0{,}s_1$};
\end{scope}
\begin{scope}
  \definecolor{strokecol}{rgb}{0.0,0.0,0.0};
  \pgfsetstrokecolor{strokecol}
  \draw (214.0bp,150.93bp) ellipse (18.0bp and 18.0bp);
  \draw (214.0bp,150.93bp) ellipse (22.0bp and 22.0bp);
  \draw (214.0bp,150.93bp) node {\LARGE $q_1{,}s_0$};
\end{scope}
\begin{scope}
  \definecolor{strokecol}{rgb}{0.0,0.0,0.0};
  \pgfsetstrokecolor{strokecol}
  \draw (214.0bp,86.93bp) ellipse (18.0bp and 18.0bp);
  \draw (214.0bp,86.93bp) ellipse (22.0bp and 22.0bp);
  \draw (214.0bp,86.929bp) node {\LARGE $q_2{,}s_0$};
\end{scope}
\begin{scope}
  \definecolor{strokecol}{rgb}{0.0,0.0,0.0};
  \pgfsetstrokecolor{strokecol}
  \draw (325.0bp,260.93bp) -- (303.0bp,260.93bp) -- (303.0bp,238.93bp) -- (325.0bp,238.93bp) -- cycle;
  \draw (314.0bp,249.93bp) node {\LARGE $\theta$};
\end{scope}
\begin{scope}
  \definecolor{strokecol}{rgb}{0.0,0.0,0.0};
  \pgfsetstrokecolor{strokecol}
  \draw (325.0bp,216.93bp) -- (303.0bp,216.93bp) -- (303.0bp,194.93bp) -- (325.0bp,194.93bp) -- cycle;
  \draw (314.0bp,205.93bp) node {\LARGE $\epsilon_1$};
\end{scope}
\begin{scope}
  \definecolor{strokecol}{rgb}{0.0,0.0,0.0};
  \pgfsetstrokecolor{strokecol}
  \draw (325.0bp,304.93bp) -- (303.0bp,304.93bp) -- (303.0bp,282.93bp) -- (325.0bp,282.93bp) -- cycle;
  \draw (314.0bp,293.93bp) node {\LARGE $\epsilon_2$};
\end{scope}
\begin{scope}
  \definecolor{strokecol}{rgb}{0.0,0.0,0.0};
  \pgfsetstrokecolor{strokecol}
  \draw (414.0bp,165.93bp) ellipse (18.0bp and 18.0bp);
  \draw (414.0bp,165.93bp) ellipse (22.0bp and 22.0bp);
  \draw (414.0bp,165.93bp) node {\LARGE $q_1{,}s_1$};
\end{scope}
\begin{scope}
  \definecolor{strokecol}{rgb}{0.0,0.0,0.0};
  \pgfsetstrokecolor{strokecol}
  \draw (414.0bp,293.93bp) ellipse (18.0bp and 18.0bp);
  \draw (414.0bp,293.93bp) ellipse (22.0bp and 22.0bp);
  \draw (414.0bp,293.93bp) node {\LARGE $q_2{,}s_1$};
\end{scope}
\begin{scope}
  \definecolor{strokecol}{rgb}{0.0,0.0,0.0};
  \pgfsetstrokecolor{strokecol}
  \draw (325.0bp,176.93bp) -- (303.0bp,176.93bp) -- (303.0bp,154.93bp) -- (325.0bp,154.93bp) -- cycle;
  \draw (314.0bp,165.93bp) node {\LARGE $\alpha$};
\end{scope}
\begin{scope}
  \definecolor{strokecol}{rgb}{0.0,0.0,0.0};
  \pgfsetstrokecolor{strokecol}
  \draw (325.0bp,136.93bp) -- (303.0bp,136.93bp) -- (303.0bp,114.93bp) -- (325.0bp,114.93bp) -- cycle;
  \draw (314.0bp,125.93bp) node {\LARGE $\beta$};
\end{scope}
\begin{scope}
  \definecolor{strokecol}{rgb}{0.0,0.0,0.0};
  \pgfsetstrokecolor{strokecol}
  \draw (525.0bp,169.93bp) -- (503.0bp,169.93bp) -- (503.0bp,147.93bp) -- (525.0bp,147.93bp) -- cycle;
  \draw (514.0bp,158.93bp) node {\LARGE $\theta$};
\end{scope}
\begin{scope}
  \definecolor{strokecol}{rgb}{0.0,0.0,0.0};
  \pgfsetstrokecolor{strokecol}
  \draw (610.0bp,75.93bp) ellipse (18.0bp and 18.0bp);
  \draw (610.0bp,75.929bp) node {\LARGE $q_3{,}s_1$};
\end{scope}
\begin{scope}
  \definecolor{strokecol}{rgb}{0.0,0.0,0.0};
  \pgfsetstrokecolor{strokecol}
  \draw (325.0bp,58.93bp) -- (303.0bp,58.93bp) -- (303.0bp,36.93bp) -- (325.0bp,36.93bp) -- cycle;
  \draw (314.0bp,47.929bp) node {\LARGE $\alpha$};
\end{scope}
\begin{scope}
  \definecolor{strokecol}{rgb}{0.0,0.0,0.0};
  \pgfsetstrokecolor{strokecol}
  \draw (369.5bp,100.93bp) -- (347.5bp,100.93bp) -- (347.5bp,78.93bp) -- (369.5bp,78.93bp) -- cycle;
  \draw (358.5bp,89.929bp) node {\LARGE $\beta$};
\end{scope}
\begin{scope}
  \definecolor{strokecol}{rgb}{0.0,0.0,0.0};
  \pgfsetstrokecolor{strokecol}
  \draw (414.0bp,50.93bp) ellipse (18.0bp and 18.0bp);
  \draw (414.0bp,50.929bp) node {\LARGE $q_3{,}s_0$};
\end{scope}
\begin{scope}
  \definecolor{strokecol}{rgb}{0.0,0.0,0.0};
  \pgfsetstrokecolor{strokecol}
  \draw (525.0bp,304.93bp) -- (503.0bp,304.93bp) -- (503.0bp,282.93bp) -- (525.0bp,282.93bp) -- cycle;
  \draw (514.0bp,293.93bp) node {\LARGE $\theta$};
\end{scope}
\begin{scope}
  \definecolor{strokecol}{rgb}{0.0,0.0,0.0};
  \pgfsetstrokecolor{strokecol}
  \draw (525.0bp,46.93bp) -- (503.0bp,46.93bp) -- (503.0bp,24.93bp) -- (525.0bp,24.93bp) -- cycle;
  \draw (514.0bp,35.929bp) node {\LARGE $\alpha$};
\end{scope}
\begin{scope}
  \definecolor{strokecol}{rgb}{0.0,0.0,0.0};
  \pgfsetstrokecolor{strokecol}
  \draw (525.0bp,86.93bp) -- (503.0bp,86.93bp) -- (503.0bp,64.93bp) -- (525.0bp,64.93bp) -- cycle;
  \draw (514.0bp,75.929bp) node {\LARGE $\beta$};
\end{scope}
\begin{scope}
  \definecolor{strokecol}{rgb}{0.0,0.0,0.0};
  \pgfsetstrokecolor{strokecol}
  \draw (717.0bp,86.93bp) -- (695.0bp,86.93bp) -- (695.0bp,64.93bp) -- (717.0bp,64.93bp) -- cycle;
  \draw (706.0bp,75.929bp) node {\LARGE $\theta$};
\end{scope}
\end{tikzpicture}

%% file: experiment.tex
\section{Implementation and Case Studies} \label{sec:case}

We implemented our RL-based synthesis framework in Python; we used Rabinizer~4~\cite{kretinsky2018} to map LTL~formulas into LDBAs, and Q-learning  %
for the proposed discounting rewards. The code and videos are available~at~\cite{csrl2019}.
We evaluated our framework on two motion planning case studies.
We consider two scenarios in a grid-world where a mobile robot can take four actions \textit{top, left, down} and \textit{right} (Fig.~\ref{fig:safe_absorbing_states} and~\ref{fig:nursery_scenario}). 
The robot moves in the intended direction with probability $0.8$ and it can go sideways with probability $0.2$ ($0.1$ each). 
If the robot hits a wall or an obstacle it stays in the~same~state. 

For Q-learning, we used $\varepsilon$-greedy policy to choose~the optimal actions, and discount factors $\gamma_B=0.99$ and $\gamma=0.99999$.
The probability that a random action is taken, $\varepsilon$, and the learning rate, $\alpha$, were gradually decreased from $1.0$ to $0.1$ and then $0.001$. %
The objective policies and estimates~of~the maximal probabilities were obtained using $100\,000$ %
episodes.

\begin{figure}[!t]
    \vspace{0pt}
    \centering
    \begin{subfigure}[b]{0.147\textwidth} %
        \centering
        \includegraphics[width=\textwidth]{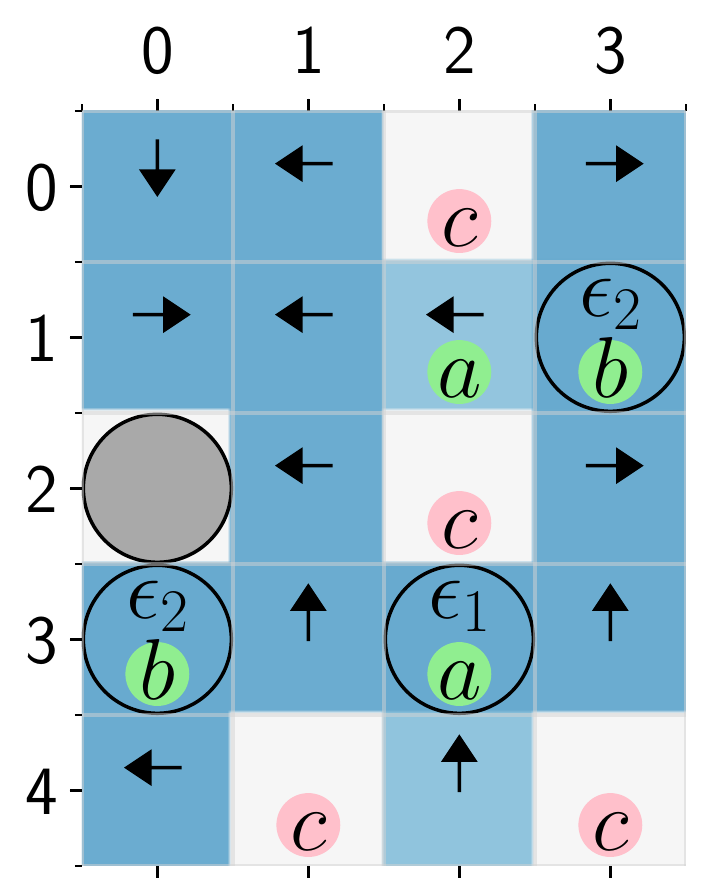}
        \caption{Policy}
        \label{fig:safe_absorbing_states_policy}
    \end{subfigure}
    \begin{subfigure}[b]{0.147\textwidth}
        \centering
        \includegraphics[width=\textwidth]{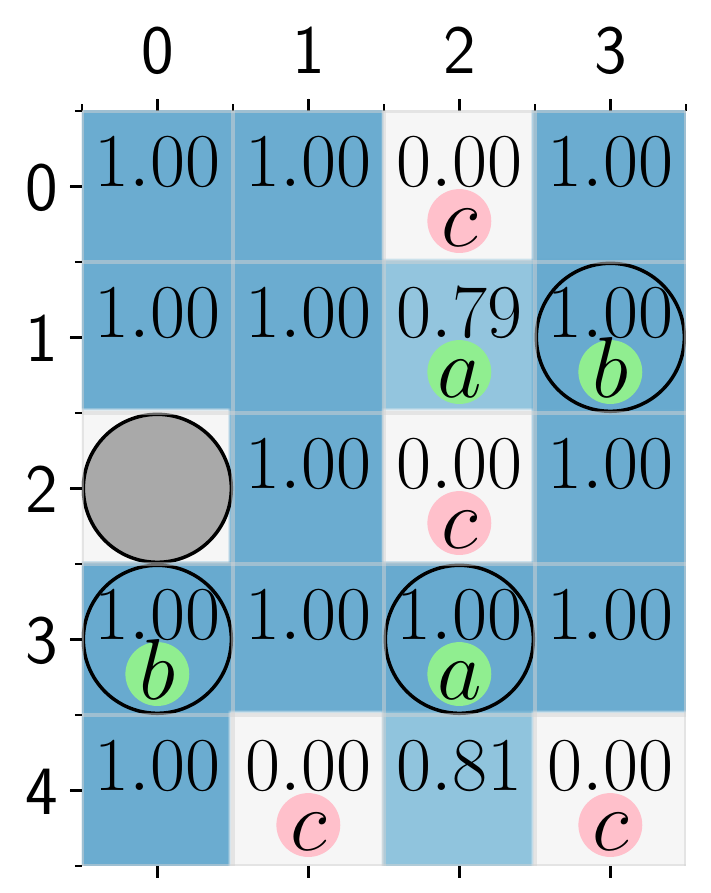}
        \caption{Value Estimates}
        \label{fig:safe_absorbing_states_values}
    \end{subfigure}
    \begin{subfigure}[b]{0.179\textwidth}
        \centering
        \includegraphics[width=\textwidth]{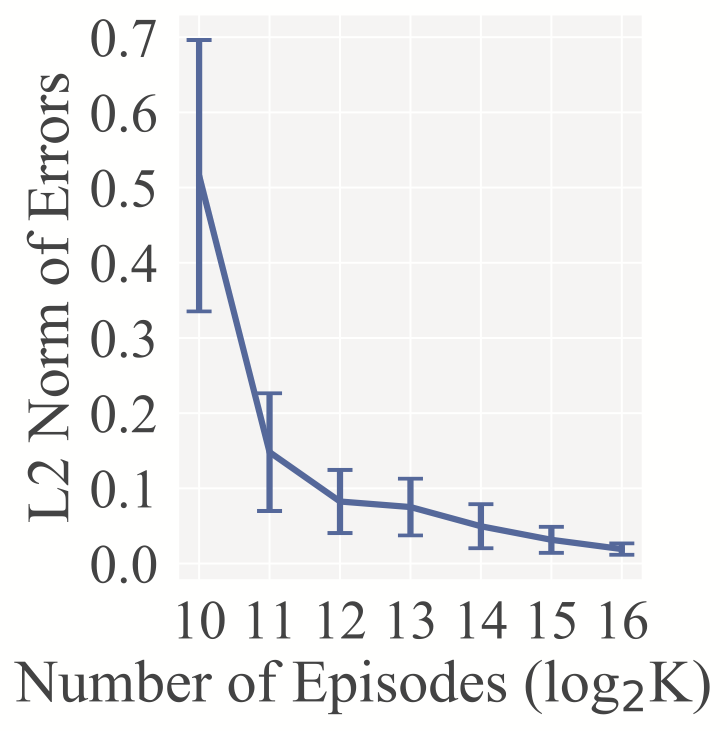}%
        \caption{Convergence}
        \label{fig:convergence}
    \end{subfigure}
    \caption{\small The objective policy and the estimated maximal probabilities of satisfying  $\varphi_1$ from~\eqref{eq:absorbing_states}. %
    Empty circles: absorbing states; Filled circles: obstacles; Arrows: actions \textit{top, left, down} and \textit{right}  and $\epsilon_1, \epsilon_2$ are $\epsilon$-actions. State labels: encircled letters in the lower part of the cells. The values are rounded to the closest hundredth.
    } 
    \vspace{-9pt}  %
    \label{fig:safe_absorbing_states}
\end{figure}

\subsection{Motion Planning with Safe Absorbing States}
\vspace{-1pt}
In this example, the robot tries to reach a safe absorbing state (states $a$ or $b$ in circle), while avoiding unsafe states (states $c$). %
This is formally specified in LTL as
\vspace{-3pt}
\begin{align}
    \varphi_1 = (\lozenge \square a \vee \lozenge \square b) \wedge \square \neg c. \label{eq:absorbing_states}
\end{align}

 \vspace{-3pt}
The LDBA computed from $\varphi_1$ has $4$ states and the product MDP has $80$ states. 
All episodes started in a random state~and were terminated after $T=100$ steps.
\

The optimal policy obtained for an MDP is illustrated in Fig.~\ref{fig:safe_absorbing_states_policy}. 
The shortest way to enter a safe absorbing state from $(0,0)$ is reaching $(1,3)$ via $(1,2)$; yet, in that case, the robot visits an unsafe state with probability 0.2. Thus,~the optimal policy tries to enter one of $(3,0)$ and $(3,2)$ by choosing \textit{up} in $(3,1)$. 
Under this policy, the robot eventually reaches a safe absorbing state without visiting an unsafe state almost surely. 
Once the robot enters an absorbing state,~it~chooses an $\epsilon$-action depending on the state label, %
and thus the LDBA transitions to an accepting state, %
with positive~rewards.

Fig.~\ref{fig:safe_absorbing_states_values} shows the estimates of the maximal probabilities. 
Note that the approximation errors in $(1,2)$ and $(4,2)$ are due to the variance of the return caused by the unsafe states. 
When the robot visits an unsafe state, the LDBA makes a transition to a trap state, making it impossible for the robot to receive a positive reward. 
Hence, the return that can be obtained from $(1,2)$ and $(4,2)$ is either 1 or 0 with probability $0.8$ and $0.2$, respectively. In addition, this type of non-0 or non-1 probability guarantees cannot be provided with existing learning-based methods for LTL specifications.

While the values from Fig.~\ref{fig:safe_absorbing_states_policy} and \ref{fig:safe_absorbing_states_values} were obtained from a single run over $K{=}100\,000$ episodes, we investigated the impact of the number of episodes. %
Fig.~\ref{fig:convergence} shows the L2 norm of the errors averaged over 100 repetitions for different numbers of episodes (the error bars show standard deviation). %

\begin{figure}[!t]
    \vspace{0pt}
    \centering
    \begin{subfigure}[b]{0.116\textwidth}
        \centering
        \includegraphics[width=\textwidth]{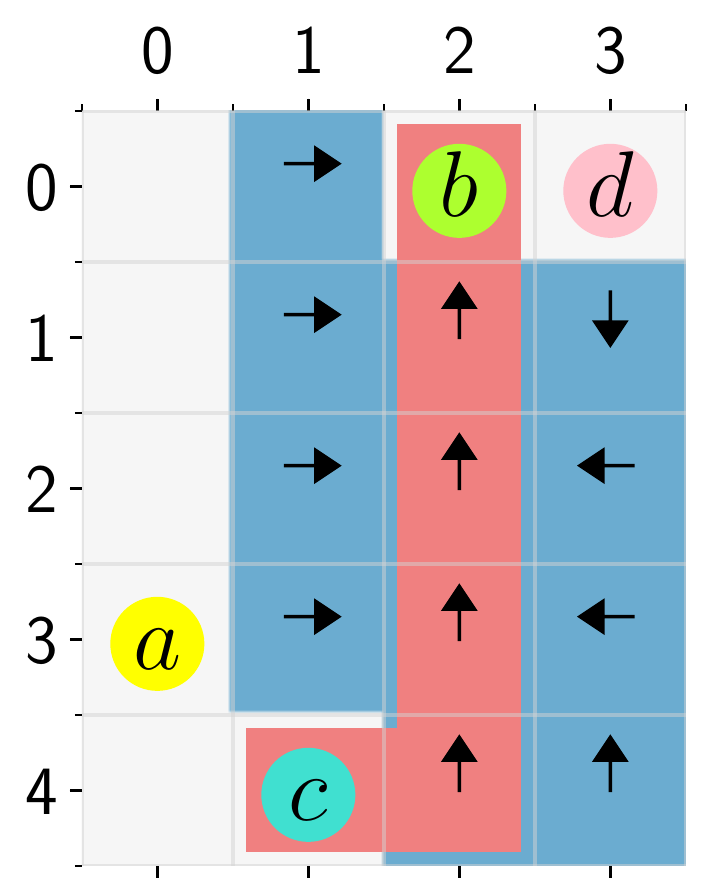}
        \caption{Policy $c$ to~$b$}
        \label{fig:nursery_scenario_policy_cb}
    \end{subfigure}
    \begin{subfigure}[b]{0.116\textwidth}
        \centering
        \includegraphics[width=\textwidth]{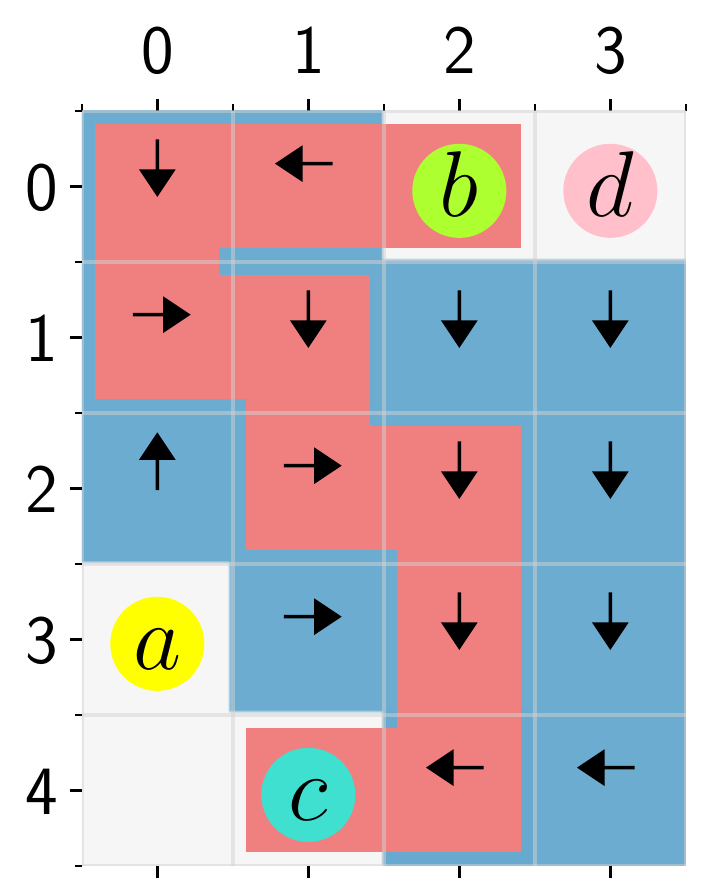}
        \caption{Policy  $b$ to~$c$}
        \label{fig:nursery_scenario_policy_bc}
    \end{subfigure}
    \begin{subfigure}[b]{0.116\textwidth}
        \centering
        \includegraphics[width=\textwidth]{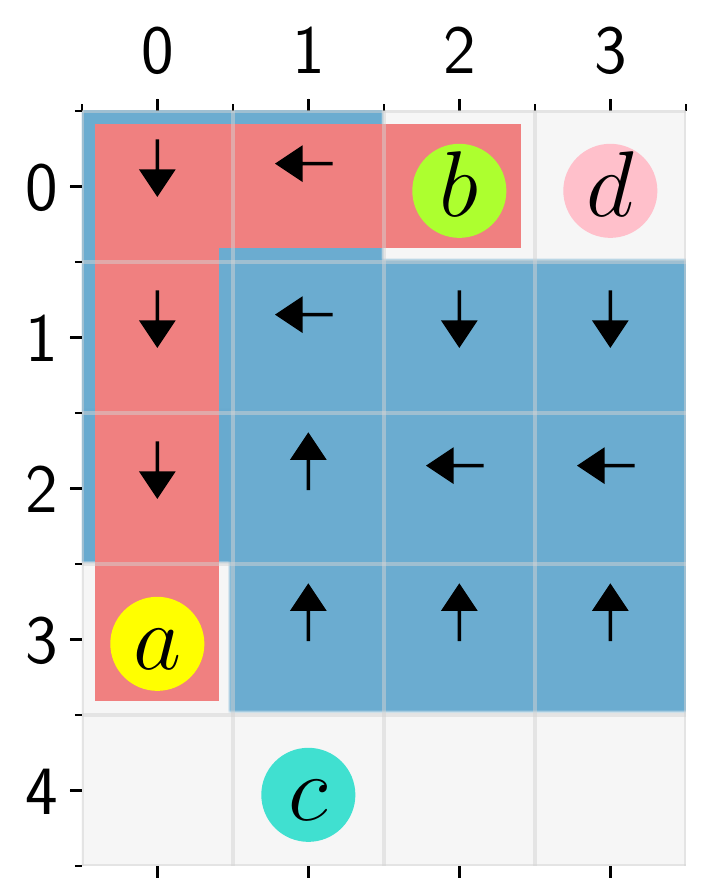}
        \caption{Policy  $b$ to~$a$}
        \label{fig:nursery_scenario_policy_ba}
    \end{subfigure}
    \begin{subfigure}[b]{0.116\textwidth}
        \centering
        \includegraphics[width=\textwidth]{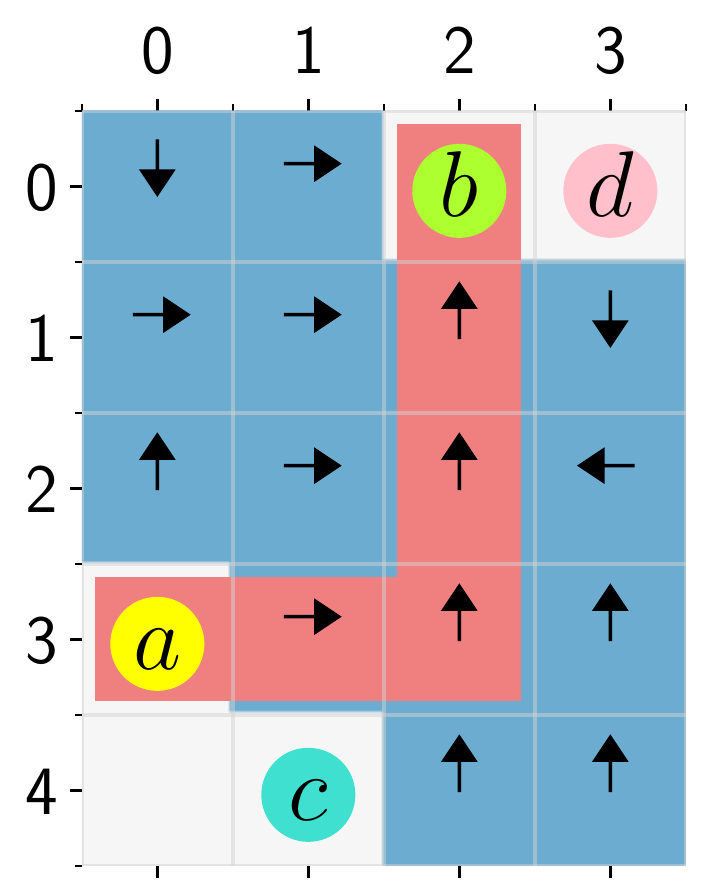}
        \caption{Policy  $a$ to~$b$}
        \label{fig:nursery_scenario_policy_ab}
    \end{subfigure}
    \caption{\small 
    A summary of the synthesized policy for the nursery scenario. Arrows: actions \textit{top, left, down}, and \textit{right}; encircled characters: state labels. The actions in states that are not reachable or lead to another LDBA state are not displayed.  In all subfigures, the most likely paths are highlighted in~red.
    }
    \vspace{-9pt}
    \label{fig:nursery_scenario}
\end{figure}

\subsection{Mobile Robot in Nursery Scenario}
\vspace{-1pt}
In this scenario (inspired by \cite{kress-gazit2007}), the robot's objective is to
repeatedly check a baby (at state $b$) and go back to its charger (at state $c$), while avoiding the danger zone (at state $d$).
Near the baby $b$, the only allowed action is \textit{left} and when taken the following situations can happen: (i) the robot hits the wall with probability $0.1$ and wakes the baby up; (ii) the robot moves left with probability $0.8$ or moves down with probability $0.1$.
If the baby has been woken up, which means the robot could not leave in a single time step (represented by LTL as $b \wedge \bigcirc b$), the robot should notify the adult (at state $a$); otherwise, the robot should directly go back to the charger (at state $c$).
The full objective is specified in LTL as

\vspace{-10pt}
\footnotesize
\[
\begin{split}
    & \varphi_2 = 
    \square \Big( \underbrace{\neg d}_{(1)} 
    \land \underbrace{ ( b \wedge \neg \bigcirc b) \to \bigcirc (\neg b\ \textsf{U}\ (a \vee c)) }_{(2)} \land \underbrace{a \to \bigcirc(\neg a\ \textsf{U}\ b) }_{(3)}
    \\ & 
    \land \underbrace{ ( \neg b \wedge \bigcirc b \wedge \neg \bigcirc\bigcirc b) \hspace{-3pt} \to \hspace{-3pt}(\neg a\ \textsf{U}\ c) }_{(4)} 
     \land \underbrace{c \hspace{-3pt}\to\hspace{-3pt} (\neg a\ \textsf{U}\ b) }_{(5)} \land \underbrace{(b \wedge \bigcirc b)\hspace{-3pt}\to\hspace{-3pt} \lozenge a }_{(6)} \Big).
\end{split}
\]
\normalsize
\vspace{-8pt}

\noindent Here, the sub-formulas mean (1)~avoid the danger state; 
(2)~if the baby is left, do not return before visiting the adult or the charger;
(3)~after notifying the adult, leave immediately and go for the baby; 
(4)~after leaving the baby sleeping, go for the charger and do not notify the adult;
(5)~after charging, return to the baby first without visiting the adult; and
(6)~notify the adult if the baby has woken up.

The LDBA for this specification has 47 states and the product MDP has 940 states. The episodes were terminated after $1000$ %
steps and the robot position was reset to~charging. %

Fig.~\ref{fig:nursery_scenario} depicts the optimal policy for the four most visited LDBA states during the simulation. The robot follows the policy in Fig.~\ref{fig:nursery_scenario_policy_cb} after it leaves the charger dock $(4,1)$. Under this policy, the robot almost surely reaches the baby in $(0,2)$, while successfully avoiding visiting $a$. Similarly, the policy in Fig.~\ref{fig:nursery_scenario_policy_bc} is followed by the robot to go back to the charger while the baby is sleeping. If the baby is awake, the robot takes the shortest path to reach $a$ (Fig.~\ref{fig:nursery_scenario_policy_ba}). %

%% file: root.bbl
\begin{thebibliography}{10}

\bibitem{karaman2011sampling}
Sertac Karaman and Emilio Frazzoli.
\newblock Sampling-based algorithms for optimal motion planning.
\newblock {\em The International Journal of Robotics Research}, 30(7):846--894,
  2011.

\bibitem{karaman2011anytime}
Sertac Karaman, Matthew~R Walter, Alejandro Perez, Emilio Frazzoli, and Seth
  Teller.
\newblock Anytime motion planning using the {RRT}.
\newblock In {\em 2011 IEEE International Conference on Robotics and
  Automation}, pages 1478--1483. IEEE, 2011.

\bibitem{vasile2013}
C.~I. {Vasile} and C.~{Belta}.
\newblock Sampling-based temporal logic path planning.
\newblock In {\em 2013 IEEE/RSJ International Conference on Intelligent Robots
  and Systems}, pages 4817--4822, Nov 2013.

\bibitem{smith2011}
Stephen~L Smith, Jana Tůmová, Calin Belta, and Daniela Rus.
\newblock Optimal path planning for surveillance with temporal-logic
  constraints.
\newblock {\em The International Journal of Robotics Research},
  30(14):1695--1708, 2011.

\bibitem{chen2012}
Y.~{Chen}, X.~C. {Ding}, A.~{Stefanescu}, and C.~{Belta}.
\newblock Formal approach to the deployment of distributed robotic teams.
\newblock {\em IEEE Transactions on Robotics}, 28(1):158--171, Feb 2012.

\bibitem{kantaros2017}
Y.~{Kantaros} and M.~M. {Zavlanos}.
\newblock Sampling-based control synthesis for multi-robot systems under global
  temporal specifications.
\newblock In {\em 2017 ACM/IEEE 8th International Conference on Cyber-Physical
  Systems (ICCPS)}, pages 3--14, April 2017.

\bibitem{wolff2014}
E.~M. {Wolff}, U.~{Topcu}, and R.~M. {Murray}.
\newblock Optimization-based trajectory generation with linear temporal logic
  specifications.
\newblock In {\em 2014 IEEE International Conference on Robotics and Automation
  (ICRA)}, pages 5319--5325, May 2014.

\bibitem{guo2018}
M.~{Guo} and M.~M. {Zavlanos}.
\newblock Probabilistic motion planning under temporal tasks and soft
  constraints.
\newblock {\em IEEE Transactions on Automatic Control}, 63(12):4051--4066, Dec
  2018.

\bibitem{guo2015multi}
Meng Guo and Dimos~V Dimarogonas.
\newblock Multi-agent plan reconfiguration under local {LTL} specifications.
\newblock {\em The International Journal of Robotics Research}, 34(2):218--235,
  2015.

\bibitem{kantaros2019}
Y.~{Kantaros} and M.~M. {Zavlanos}.
\newblock Sampling-based optimal control synthesis for multirobot systems under
  global temporal tasks.
\newblock {\em IEEE Transactions on Automatic Control}, 64(5):1916--1931, May
  2019.

\bibitem{lahijanian2012}
M.~{Lahijanian}, S.~B. {Andersson}, and C.~{Belta}.
\newblock Temporal logic motion planning and control with probabilistic
  satisfaction guarantees.
\newblock {\em IEEE Transactions on Robotics}, 28(2):396--409, April 2012.

\bibitem{wolff2012}
E.~M. {Wolff}, U.~{Topcu}, and R.~M. {Murray}.
\newblock Robust control of uncertain {M}arkov decision processes with temporal
  logic specifications.
\newblock In {\em 2012 IEEE 51st IEEE Conference on Decision and Control
  (CDC)}, pages 3372--3379, Dec 2012.

\bibitem{kwiatkowska2013}
Marta Kwiatkowska and David Parker.
\newblock Automated verification and strategy synthesis for probabilistic
  systems.
\newblock In Dang Van~Hung and Mizuhito Ogawa, editors, {\em Automated
  Technology for Verification and Analysis}, pages 5--22, Cham, 2013. Springer
  International Publishing.

\bibitem{ding2014}
X.~{Ding}, S.~L. {Smith}, C.~{Belta}, and D.~{Rus}.
\newblock Optimal control of {M}arkov decision processes with linear temporal
  logic constraints.
\newblock {\em IEEE Transactions on Automatic Control}, 59(5):1244--1257, May
  2014.

\bibitem{baier2008}
Christel Baier and Joost-Pieter Katoen.
\newblock {\em Principles of Model Checking}.
\newblock MIT Press, Cambridge, MA, USA, 2008.

\bibitem{fu2014}
Jie Fu and Ufuk Topcu.
\newblock Probably approximately correct {MDP} learning and control with
  temporal logic constraints, 2014.
\newblock arXiv:1404.7073 [cs.SY].

\bibitem{brazdil2014}
Tom{\'a}{\v{s}} Br{\'a}zdil, Krishnendu Chatterjee, Martin Chmel{\'i}k,
  Vojt{\v{e}}ch Forejt, Jan K{\v{r}}et{\'i}nsk{\'y}, Marta Kwiatkowska, David
  Parker, and Mateusz Ujma.
\newblock Verification of {M}arkov decision processes using learning
  algorithms.
\newblock In Franck Cassez and Jean-Fran{\c{c}}ois Raskin, editors, {\em
  Automated Technology for Verification and Analysis}, pages 98--114, Cham,
  2014. Springer International Publishing.

\bibitem{wen2015}
Min Wen, R{\"u}diger Ehlers, and Ufuk Topcu.
\newblock Correct-by-synthesis reinforcement learning with temporal logic
  constraints.
\newblock In {\em 2015 IEEE/RSJ International Conference on Intelligent Robots
  and Systems (IROS)}, pages 4983--4990. IEEE, 2015.

\bibitem{aksaray2016}
D.~{Aksaray}, A.~{Jones}, Z.~{Kong}, M.~{Schwager}, and C.~{Belta}.
\newblock Q-learning for robust satisfaction of signal temporal logic
  specifications.
\newblock In {\em 2016 IEEE 55th Conference on Decision and Control (CDC)},
  pages 6565--6570, Dec 2016.

\bibitem{li2017}
X.~{Li}, C.~{Vasile}, and C.~{Belta}.
\newblock Reinforcement learning with temporal logic rewards.
\newblock In {\em 2017 IEEE/RSJ International Conference on Intelligent Robots
  and Systems (IROS)}, pages 3834--3839, Sep. 2017.

\bibitem{toro2018}
Rodrigo Toro~Icarte, Toryn~Q Klassen, Richard Valenzano, and Sheila~A
  McIlraith.
\newblock Teaching multiple tasks to an {RL} agent using {LTL}.
\newblock In {\em Proceedings of the 17th International Conference on
  Autonomous Agents and MultiAgent Systems}, pages 452--461. International
  Foundation for Autonomous Agents and Multiagent Systems, 2018.

\bibitem{degiacomo2019}
Giuseppe De~Giacomo, Luca Iocchi, Marco Favorito, and Fabio Patrizi.
\newblock Foundations for restraining bolts: Reinforcement learning with
  {LTL}$_f$/{LDL}$_f$ restraining specifications.
\newblock In {\em Proceedings of the International Conference on Automated
  Planning and Scheduling}, volume~29, pages 128--136, 2019.

\bibitem{sadigh2014}
D.~{Sadigh}, E.~S. {Kim}, S.~{Coogan}, S.~S. {Sastry}, and S.~A. {Seshia}.
\newblock A learning based approach to control synthesis of {M}arkov decision
  processes for linear temporal logic specifications.
\newblock In {\em 53rd IEEE Conference on Decision and Control}, pages
  1091--1096, Dec 2014.

\bibitem{gao2019}
Qitong Gao, Davood Hajinezhad, Yan Zhang, Yiannis Kantaros, and Michael~M.
  Zavlanos.
\newblock Reduced variance deep reinforcement learning with temporal logic
  specifications.
\newblock In {\em Proceedings of the 10th ACM/IEEE International Conference on
  Cyber-Physical Systems}, ICCPS '19, pages 237--248, New York, NY, USA, 2019.
  ACM.

\bibitem{hahn2019}
Ernst~Moritz Hahn, Mateo Perez, Sven Schewe, Fabio Somenzi, Ashutosh Trivedi,
  and Dominik Wojtczak.
\newblock Omega-regular objectives in model-free reinforcement learning.
\newblock In Tom{\'a}{\v{s}} Vojnar and Lijun Zhang, editors, {\em Tools and
  Algorithms for the Construction and Analysis of Systems}, pages 395--412,
  Cham, 2019. Springer International Publishing.

\bibitem{li2018policy}
Xiao Li, Yao Ma, and Calin Belta.
\newblock A policy search method for temporal logic specified reinforcement
  learning tasks.
\newblock In {\em 2018 Annual American Control Conference (ACC)}, pages
  240--245. IEEE, 2018.

\bibitem{hahn2015}
Ernst~Moritz Hahn, Guangyuan Li, Sven Schewe, Andrea Turrini, and Lijun Zhang.
\newblock {Lazy Probabilistic Model Checking without Determinisation}.
\newblock In Luca Aceto and David de~Frutos~Escrig, editors, {\em 26th
  International Conference on Concurrency Theory (CONCUR 2015)}, volume~42 of
  {\em Leibniz International Proceedings in Informatics (LIPIcs)}, pages
  354--367, Dagstuhl, Germany, 2015. Schloss Dagstuhl--Leibniz-Zentrum fuer
  Informatik.

\bibitem{sickert2016}
Salomon Sickert, Javier Esparza, Stefan Jaax, and Jan K{\v{r}}et{\'i}nsk{\'y}.
\newblock Limit-deterministic {B}{\"u}chi automata for linear temporal logic.
\newblock In Swarat Chaudhuri and Azadeh Farzan, editors, {\em Computer Aided
  Verification}, pages 312--332, Cham, 2016. Springer International Publishing.

\bibitem{hasanbeig2018}
Mohammadhosein Hasanbeig, Alessandro Abate, and Daniel Kroening.
\newblock Logically-constrained reinforcement learning.
\newblock {\em arXiv:1801.08099 [cs.LG]}, 2018.

\bibitem{strehl2006}
Alexander~L. Strehl, Lihong Li, Eric Wiewiora, John Langford, and Michael~L.
  Littman.
\newblock Pac model-free reinforcement learning.
\newblock In {\em Proceedings of the 23rd International Conference on Machine
  Learning}, ICML ’06, page 881–888, New York, NY, USA, 2006. Association
  for Computing Machinery.

\bibitem{sutton2018}
Richard~S Sutton and Andrew~G Barto.
\newblock {\em Reinforcement Learning: An Introduction}.
\newblock MIT Press, Cambridge, MA, USA, 2nd edition, 2018.

\bibitem{jaakkola1994}
Tommi Jaakkola, Michael~I. Jordan, and Satinder~P. Singh.
\newblock Convergence of stochastic iterative dynamic programming algorithms.
\newblock In J.~D. Cowan, G.~Tesauro, and J.~Alspector, editors, {\em Advances
  in Neural Information Processing Systems 6}, pages 703--710. Morgan-Kaufmann,
  1994.

\bibitem{kretinsky2018}
Jan K{\v{r}}et{\'i}nsk{\'y}, Tobias Meggendorfer, Salomon Sickert, and
  Christopher Ziegler.
\newblock Rabinizer 4: From {LTL} to your favourite deterministic automaton.
\newblock In Hana Chockler and Georg Weissenbacher, editors, {\em Computer
  Aided Verification}, pages 567--577, Cham, 2018. Springer International
  Publishing.

\bibitem{csrl2019}
A.~K. Bozkurt, Y.~Wang, M.~M. Zavlanos, and M.~Pajic.
\newblock {CSRL}, 2019.
\newblock \url{https://gitlab.oit.duke.edu/cpsl/csrl}.

\bibitem{kress-gazit2007}
H.~{Kress-Gazit}, G.~E. {Fainekos}, and G.~J. {Pappas}.
\newblock Where's {W}aldo? sensor-based temporal logic motion planning.
\newblock In {\em Proceedings 2007 IEEE International Conference on Robotics
  and Automation}, pages 3116--3121, April 2007.

\end{thebibliography}
